\documentclass[11pt,dvipsnames]{article}

\usepackage[font={scriptsize,sf}]{caption}
\usepackage{sansmath}

\usepackage{titling}

\usepackage{times}
\usepackage[T1]{fontenc}
\usepackage[margin=1in]{geometry}
\usepackage[utf8]{inputenc}

\usepackage{amsmath,amssymb,amsthm,thmtools,mathtools}
\usepackage{graphicx}
\usepackage{xcolor,colortbl}

\usepackage{sansmath}
\newcommand{\s}{\fontsize{7pt}{8pt}\selectfont}
\newcommand{\sm}{\fontsize{5pt}{8pt}\selectfont}

\usepackage{lmodern}

\usepackage{listings}
\usepackage{mdframed}

\newtheorem{proposition}{Proposition}

\newtheorem{remark}{Remark}

\newcounter{rcodeno}
\newcommand{\rcode}[1]{\refstepcounter{rcodeno}\label{#1}}

\newcommand{\changelog}{{}}

\newcommand{\actMNash}{-0.2}

\newcommand{\R}{\mathbb{R}}

\newcommand{\mc}{\mathcal}
\DeclareMathOperator*{\amin}{\text{argmin}}

\newcommand{\eqn}[1]{\begin{equation*}\begin{aligned}#1\end{aligned}\end{equation*}}

\newcommand{\grad}{\partial}

\newcommand{\ActH}{\mathcal{H}}
\newcommand{\ActM}{\mathcal{M}}
\newcommand{\costH}{c_H}
\newcommand{\costM}{c_M}
\newcommand{\actH}{h}
\newcommand{\actM}{m}

\newcommand{\polM}{\pi_M}
\newcommand{\Conj}{L}

\newcommand{\ConjH}{L_H}
\newcommand{\conjH}{\ell_H}
\newcommand{\Hconj}{v_H}
\newcommand{\ConjM}{L_M}
\newcommand{\conjM}{\ell_M}
\newcommand{\Mconj}{v_M}
\newcommand{\NE}{\text{NE}}
\newcommand{\SE}{\text{SE}}
\newcommand{\RSE}{\text{RSE}}
\newcommand{\CVE}{\text{CVE}}
\newcommand{\CCVE}{\text{CCVE}}
\newcommand{\Nash}[1]{#1^{\NE}}
\newcommand{\Opt}[1]{#1^{*}}
\newcommand{\Stack}[1]{#1^{\SE}}
\newcommand{\RStack}[1]{#1^{\RSE}}
\newcommand{\Cve}[1]{#1^{\CVE}}
\newcommand{\Ccve}[1]{#1^{\CCVE}}

\newcommand{\eqnn}[2]{\begin{equation}\label{eq:#1}\begin{aligned}#2\end{aligned}\end{equation}}
\newcommand{\eqnns}[2]{\begin{subequations}\label{eq:#1}\begin{align}#2\end{align}\end{subequations}}
\newcommand{\paren}[1]{\left( #1 \right)}

\newcommand{\set}[1]{\left\{ #1 \right\}}
\newcommand{\st}{\mid}
\newcommand{\est}[1]{\widetilde{#1}}

\newcommand{\HA}{A_H}
\newcommand{\HB}{B_H}
\newcommand{\HD}{D_H}
\newcommand{\Hb}{b_H}
\newcommand{\Hd}{d_H}
\newcommand{\MA}{A_M}
\newcommand{\MB}{B_M}
\newcommand{\MD}{D_M}
\newcommand{\Mb}{b_M}
\newcommand{\Md}{d_M}
\newcommand{\figref}[1]{Figure~\ref{fig:#1}}
\newcommand{\subfigref}[2]{Figure~\ref{fig:#1}#2}
\newcommand{\secref}[1]{Section~\ref{sec:#1}}

\newcommand{\citedef}[2]{(Definition~#1 in~\cite{#2})}

\newcommand{\learningrate}{adaptation rate}
\newcommand{\optimum}{optimum}
\newcommand{\optima}{optima}

\newcommand{\into}{\rightarrow}

\newcommand{\tops}{{}} 

\usepackage{lineno}
\definecolor{lightgray}{rgb}{0.7,0.7,0.7}

\usepackage[normalem]{ulem}

\usepackage[sort&compress,round]{natbib}
\usepackage{xcolor}
\usepackage{hyperref}
\hypersetup{
    colorlinks=true,
    linkcolor=OrangeRed,
    filecolor=magenta,      
    urlcolor=cyan,
    citecolor=RoyalBlue, 
    }
\let\cite\citep

\title{Human adaptation to adaptive machines \\ converges to game-theoretic equilibria}

\author {
Benjamin J. Chasnov, Lillian J.\ Ratliff, Samuel A.\ Burden\\ 
\\
\footnotesize 
Department of Electrical \& Computer Engineering\\
\footnotesize
University of Washington, Seattle, WA 98195, USA
}

\date{}

\begin{document}

\maketitle

\begin{abstract}
Adaptive machines have the potential to assist \emph{or} interfere with \mbox{human} behavior in a range of contexts, from cognitive decision-making~\cite{sutton2020overview,mehrabi2021survey} to physical device assistance~\cite{Felt2015-qc,Zhang2017-jq,slade2022personalizing}.
Therefore it is critical to understand 
how machine learning algorithms
can influence human actions, particularly in situations where machine goals are misaligned with those of people~\cite{Thomas2019-qt}.
Since humans continually adapt to their environment using a combination of explicit and implicit strategies~\cite{taylor2014explicit,heald2021contextual}, 
when the environment contains an adaptive machine, the human and machine play a \emph{game}~\cite{von1947theory,bacsar1998dynamic}.
Game theory is an established
framework for modeling interactions between two or more decision-makers that has been applied extensively in economic markets~\cite{varian1992microeconomic}
and machine algorithms~\cite{Goodfellow2014-yh}.
However, existing approaches make assumptions about, rather than empirically test, how adaptation by individual humans is affected by interaction with an adaptive machine~\cite{nikolaidis2017game,madduri2021game}.
Here we tested learning algorithms
for machines playing general-sum games with human \mbox{subjects}.
Our algorithms enable the machine to select the outcome of the co-adaptive interaction 
from a constellation of game-theoretic equilibria in action and policy spaces.
Importantly, the machine learning algorithms work directly from observations of human actions 
without solving an inverse problem 
to estimate the human's utility function as in prior work~\cite{li2019differential,Ng2000-aj}.
\mbox{Surprisingly}, one algorithm can steer the human-machine interaction to the machine's optimum, effectively controlling the human's actions even while the human responds optimally to their perceived cost landscape.
Our results show that game theory can be used to predict and design outcomes of co-adaptive interactions between intelligent humans and machines.
\end{abstract}

%

We studied games played between humans $H$ and machines $M$. The games were defined by quadratic functions that mapped scalar actions of each human $h$ and machine $m$ to costs $c_H(h,m)$ and $c_M(h,m)$.
Games were played 
continuously in time over a sequence of trials, 
and the machine adapted within or between trials.
Human actions $h$ were determined from a manual input device (mouse or touchscreen) as in \figref{exp1}{a}, while machine actions $m$ were determined algorithmically from the machine's cost function $\costM$ and the human's action $h$ as in \figref{exp1}{b}. 
The human's cost $c_H(h,m)$
was continuously shown to the human subjects via the height of a rectangle on a computer display as in \figref{exp1}{a}, 
which the subject was instructed to ``make as small as possible'', while the machine's actions were hidden.

\subsection*{Game-theoretic equilibria}
The experiments reported here were based on a game that is \emph{general-sum}, meaning that the cost functions prescribed to the human and machine were neither aligned nor opposed. 
There is no single ``solution'' concept for general-sum games -- unlike pure optimization problems,
players do not get to choose all decision variables that determine their cost.
Although each player seeks its own preferred outcome, 
the game outcome will generally represent a compromise between players' conflicting goals.
We considered 
\emph{Nash}~\cite{Nash1950-aq}, \emph{Stackelberg}~\cite{Von_Stackelberg1934-xc}, \emph{consistent conjectural variations}~\cite{bowley1924mathematical},
and 
\emph{reverse Stackelberg}~\cite{Ho1982-eu}
equilibria of the game (Definitions~4.1,~4.6,~4.9,~7.1 in~\citet{bacsar1998dynamic} respectively), in addition to each player's \emph{global {\optimum}}, as possible outcomes in the experiments. 
Formal definitions of these game-theoretic concepts are provided in \secref{theory-def} of the Supplement, but we provide plain-language descriptions in the next paragraph.
Table~1 contains expressions for the cost functions that defined the game considered here as well as numerical values of the resulting game-theoretic equilibria.

Nash equilibria~\cite{Nash1950-aq} arise in games with simultaneous play, and constitute  points in the joint action space from which neither player is incentivized to deviate (see Section~4.2 in~\citet{bacsar1998dynamic}).
In games with ordered play where one player (the \emph{leader}) chooses its action assuming the other (the \emph{follower}) will play using its best response,
a Stackelberg equilibrium~\cite{Von_Stackelberg1934-xc} may arise instead.
The leader in this case employs a \emph{conjecture} about the follower's {policy}, i.e.\ a function from the leader's actions to the follower's actions, and this conjecture is consistent with how the follower plays the game (Section~4.5 in~\citet{bacsar1998dynamic});
the leader's conjecture can be regarded as an \emph{internal model}~\cite{Huang2018-nh,Wolpert1995-cr,nikolaidis2017game} for the follower.
Shifting from Nash to Stackelberg equilibria in our quadratic setting is generally in favor of the leader whose cost decreases.
Of course, the follower may then form a conjecture of its own about the leader's play, and the players may iteratively update their policies and conjectures in response to their opponent's play. 
In the game we consider, this iteration converges to a
\emph{consistent} conjectural variations equilibrium~\cite{bowley1924mathematical}
defined in terms of actions \emph{and} conjectures: 
each player's conjecture is 
equal to
their opponent's policy, and each player's policy is optimal with respect to its conjecture about the opponent (Section~7.1 in~\citet{bacsar1998dynamic}).
Finally, if one player realizes how their choice of policy influences the other, they can design an \emph{incentive} 
to steer the game to their preferred outcome, termed a \emph{reverse} Stackelberg equilibrium~\cite{Ho1982-eu}
(Section~7.4.4 in~\citet{bacsar1998dynamic}).

\subsection*{Experimental results}
We conducted three 
experiments with different populations of human subjects
using a pair of quadratic cost functions $\costH$, $\costM$ illustrated in~\figref{exp1}a,b that were  designed to yield distinct game-theoretic equilibria in both action and policy spaces.
These analytically-determined equilibria were compared with the empirical distributions of actions and policies 
reached by humans and machines over a sequence of trials in each experiment.
In all three experiments, we found that empirically-measured actions or policies converged to their predicted game-theoretic values.

In our first experiment (\figref{exp1}), the machine adapted its action within trials using what is arguably the simplest optimization scheme: gradient descent~\cite{Ma2019-bh,Chasnov2020-yo}.
We tested seven adaptation rates $\alpha \geq 0$ for the gradient descent algorithm as illustrated in~\figref{exp1}{c,d,e} for each human subject, with two repetitions for each rate and the sequence of rates occurring in random order.
We found that distributions of median action vectors 
for the population of $n = 20$ human subjects in this experiment shifted from the \emph{Nash equilibrium} (NE) at the slowest adaptation rate
to the \emph{human-led Stackelberg equilibrium} (SE) at the fastest adaptation rate 
(\figref{exp1}{c}). 
Importantly, this result would not have obtained if the human was also adapting its action using 
gradient descent, as merely changing adaptation rates in simultaneous gradient play does not change stationary points~\cite{Chasnov2020-yo}.
The shift we observed from Nash to Stackelberg, 
which was in favor of the human (\figref{exp1}e), 
was statistically significant in that the distribution of actions was distinct from SE but not NE at the slowest adaptation rate and vice-versa for the fastest rate 
(\figref{exp1}d; $^*P \leq 0.05$; two-sided $t$-tests, degrees of freedom (df) $19$; exact statistics in Table~\ref{tab:pvalues}).
Discovering that the human's empirical play is consistent with the theoretically-predicted best-response function for its prescribed cost is important, as this insight motivated us in subsequent experiments to elevate the machine's play beyond the action space to reason over its space of \emph{policies}, that is, functions from human actions to machine actions.

\begin{figure}[t]
\centering
\includegraphics[width=\textwidth]{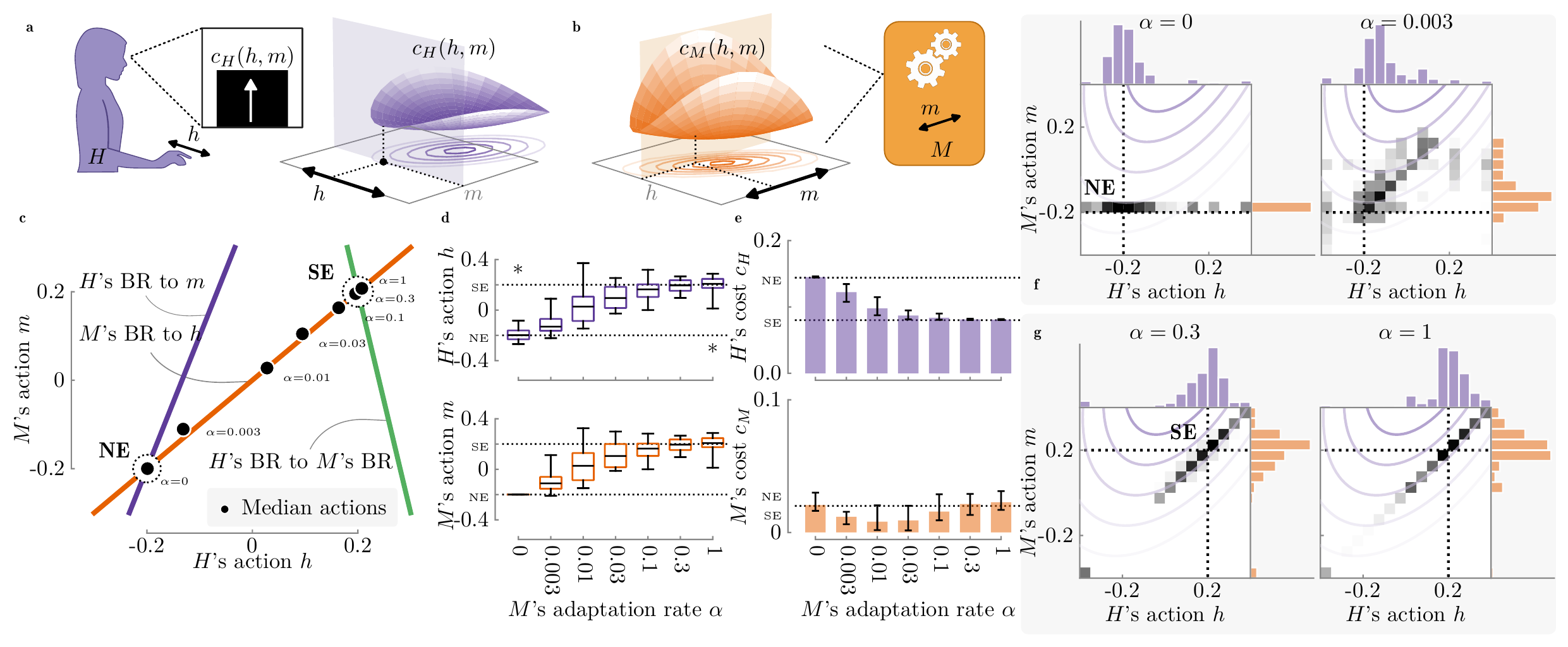}
  \caption[Experiment 1]{\label{fig:exp1}
    \textbf{Gradient descent in action space (Experiment 1, $n = 20$).}
    (\textbf{a}) Each human subject $H$ is instructed to provide manual input $h$ to make a black bar on a computer display as small as possible. The bar's height represents the value of a prescribed cost $c_H$.
    (\textbf{b}) The machine $M$ has its own cost $c_M$ chosen to yield game-theoretic equilibria that are distinct from each other and from each player's global optima. The machine knows its cost and observes human actions $h$.
    In this experiment, the machine updates its action by gradient descent on its cost $\frac{1}{2}m^2-hm+h^2$ with adaptation rate $\alpha$.
    (\textbf{c}) Median joint actions for each machine adaptation rate $\alpha$ overlaid on game-theoretic equilibria and best-response (BR) curves that define the Nash and Stackelberg equilibria (NE and SE, respectively).
    (\textbf{d}) Action distributions for each machine adaptation rate displayed by box-and-whiskers plots showing 5th, 25th, 50th, 75th, and 95th percentiles.
    Statistical significance ($*$) 
    determined by 
    comparing 
    to NE (shown below distributions) and SE (shown above distributions) 
    using two-sided $t$-tests ($^*P \leq 0.05$).
    (\textbf{e}) Cost distributions for each machine adaptation rate displayed using box plots with error bars showing 25th, 50th, and 75th percentiles. 
    (\textbf{f,g}) One- and two-dimensional histograms of actions for different adaptation rates ($\alpha \in \set{\text{0,0.003}}$ in (f), $\alpha\in\set{\text{0.3, 1}}$ in (g)) with game-theoretic equilibria overlaid (NE in (f), SE in (g)).
}
\end{figure}

\begin{figure}[t]
\centering
\includegraphics[width=\textwidth]{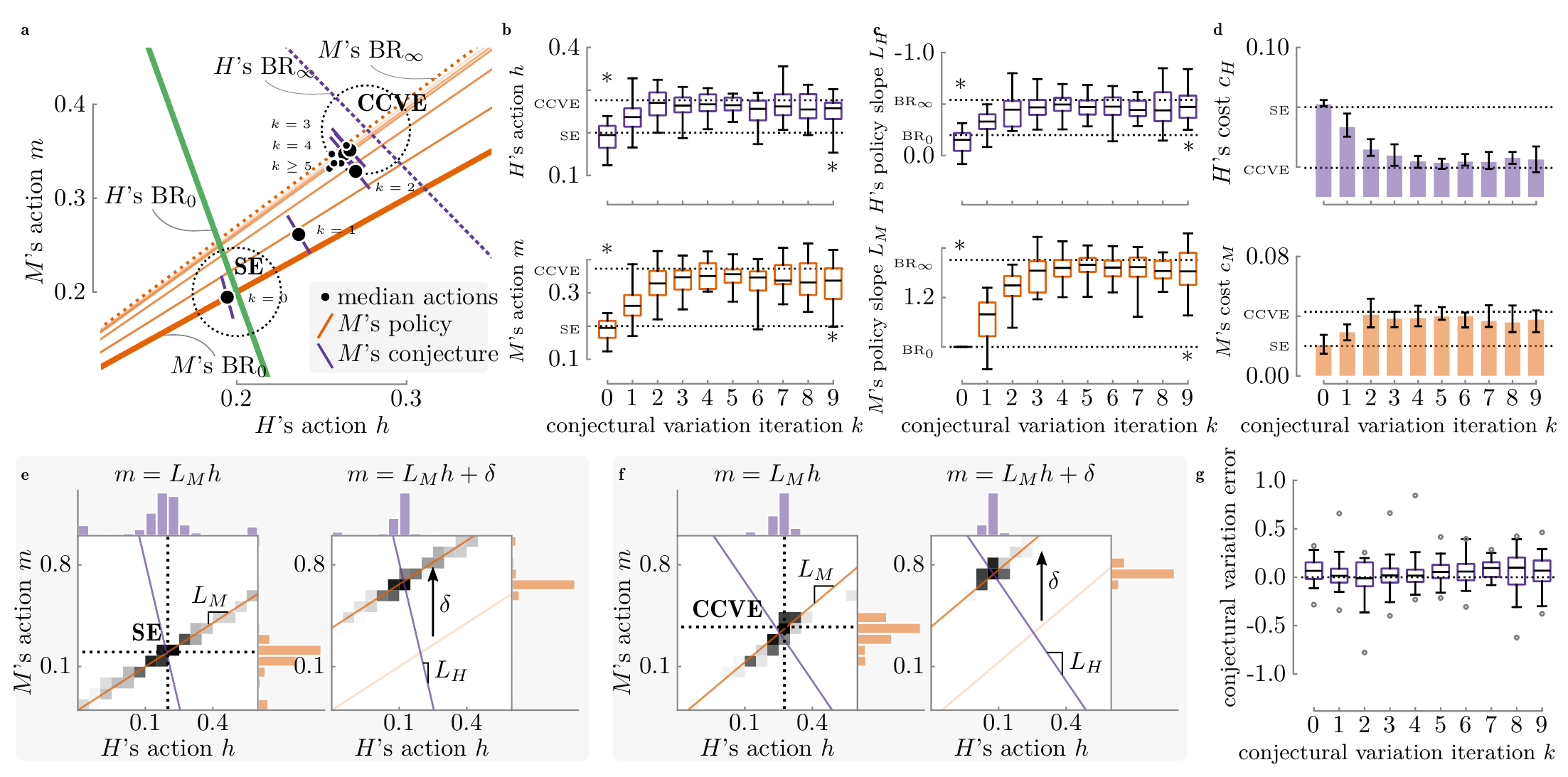}
  \caption[Experiment 2]{\label{fig:exp2}
  \textbf{Conjectural variation in policy space (Experiment 2, $n = 20$).}
  Experimental setup and costs are the same as~\figref{exp1}{a,b}
  except that the machine uses a different adaptation algorithm:
  in this experiment
  $M$ iteratively implements and updates affine policies $m = L_M h$, $m = L_M + \delta$ to measure and best-respond to conjectures of the human's policy.
  (\textbf{a}) Median actions, conjectures, and policies for each conjectural variation iteration $k$ overlaid on game-theoretic equilibria corresponding to best-responses (BR) at initial and limiting iterations (BR$_0$ and BR$_\infty$, respectively) predicted from Stackelberg and Consistent Conjectural Variations equilibria of the game (SE and CCVE), respectively. 
  (\textbf{b}) Action distributions for each iteration displayed by box-and-whiskers plots as in~\figref{exp1}d, 
  with statistical significance ($*$) 
  analogously determined using the same tests by 
  comparing 
  to SE (shown below distributions) and CCVE (above).
  (\textbf{c}) Policy slope distributions for each iteration displayed with the same conventions as (b); 
  note that the sign of the top $y$-axis is reversed for consistency with other plots.
  Statistical significance ($*$) 
  determined as in (b) by 
  comparing 
  to initial (shown below distributions) and limiting (above) best-responses 
  using two-sided $t$-tests ($^*P \leq 0.05$).
  (\textbf{d}) Cost distributions 
  for each iteration displayed using box-and-whiskers plots as in~\figref{exp1}e.
  (\textbf{e,f}) One- and two-dimensional histograms of actions for different iterations ($k=0$ in (e), $k=9$ in (f)) with policies and game-theoretic equilibria overlaid (SE and BR$_0$ in (e), CCVE and BR$_\infty$ in (f)).
  (\textbf{g})
  Error between measured 
  and theoretically-predicted 
  machine conjectures about human policies at each iteration displayed as box-and-whiskers plots as in (b,c).
}
\end{figure}

In our second experiment (\figref{exp2}), the machine played affine policies (i.e.\ $\actM$ was determined as an affine function of $\actH$) 
and adapted its policies by observing the human's response.
Trials 
came in pairs, with the machine's policy in each pair 
differing only in the constant term. 
After each pair of trials,
the machine used the median action vectors 
from the pair to estimate a \emph{conjecture}~\cite{bowley1924mathematical,Figuieres2004-ls} (or \emph{internal model}~\cite{nikolaidis2017game,Huang2018-nh,Wolpert1995-cr}) about the human's policy,
and the machine's policy 
was updated to be optimal with respect to this conjecture.
Unsurprisingly, the human adapted its own policy in response. 
Iterating this process shifted the distribution of median action vectors for a population of $n = 20$ human subjects (distinct from the population in the first experiment) from the \emph{human-led Stackelberg equilibrium} (SE) toward a
\emph{consistent conjectural variations equilibrium} (CCVE) in action and policy spaces (\subfigref{exp2}{a}).
\changelog{
The shift we observed away from SE toward CCVE from the first to last iteration 
was statistically significant in policy space
(\figref{exp2}c; $^*P \leq 0.05$; two-sided $t$-tests, degrees of freedom (df) $19$; exact statistics in Table~\ref{tab:pvalues})
but \emph{not} action space 
(\figref{exp2}b; $^*P \leq 0.05$; two-sided $t$-tests, df $19$; exact statistics in Table~\ref{tab:pvalues}).
}%
This shift was in favor of the human at the machine's expense (\figref{exp2}d).
The machines' empirical conjectures were not significantly different from theoretical predictions of human policies at all conjectural variation iterations 
(\figref{exp2}g; $P > 0.05$; two-sided $t$-tests, df $19$; exact statistics in Table~\ref{tab:pvalues}),
suggesting that both humans and machines estimated consistent conjectures of their opponent.

\begin{figure}[t]
\centering
\includegraphics[width=\textwidth]{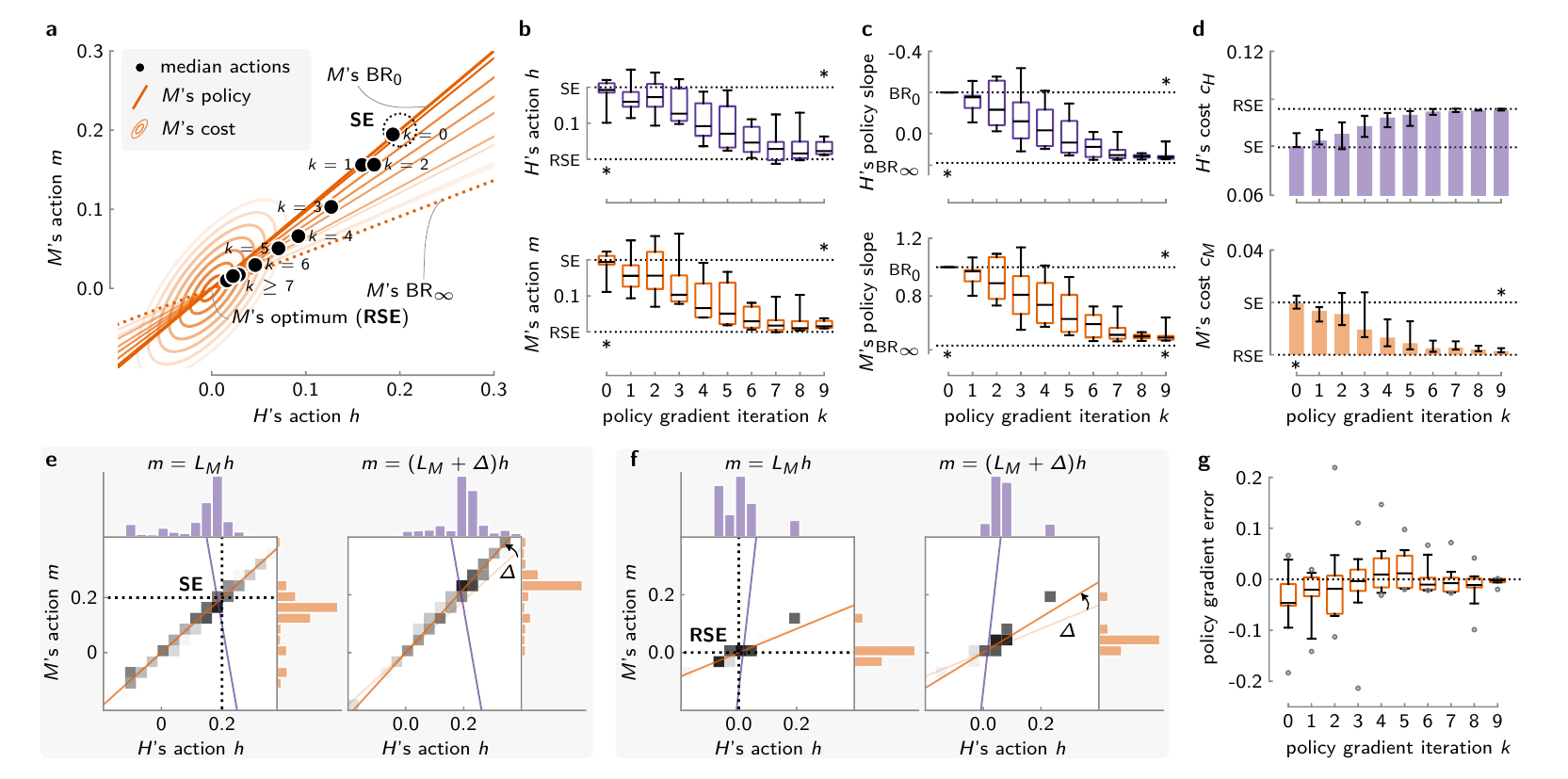}
  \caption[Experiment 3]{\label{fig:exp3}
  \textbf{Gradient descent in policy space (Experiment 3, $n = 20$).} 
  Experimental setup and costs are the same as~\figref{exp1}a,b except that the machine uses a different adaptation algorithm:
  in this experiment, $M$ iteratively implements linear policies $m = L_M h$, $m = (L_M + \Delta)h$ to measure the gradient of its cost with respect to its policy slope parameter $L_M$ and updates this parameter to descend its cost landscape.
  (\textbf{a}) 
  Median actions and policies for each policy gradient iteration $k$ overlaid on game-theoretic equilibria corresponding to machine best-responses (BR) at initial and limiting iterations (BR$_0$ and BR$_\infty$, respectively) predicted from the Stackelberg equilibrium (SE) and the machine's global optimum (RSE), respectively.
  (\textbf{b}) 
  Action distributions for each iteration displayed by box-and-whiskers plots as in~\figref{exp1}d, with statistical significance ($*$) analogously determined using the same tests by comparing to SE (shown above distributions) and $M$'s optimum (shown below distributions)
  using two-sided $t$-tests ($^*P \leq 0.05$);
  (\textbf{c}) 
  Policy slope distributions for each iteration displayed with the same conventions as (b); 
  note that the sign of the top subplot's $y$-axis is reversed for consistency with other plots.
  Statistical significance ($*$) determined as in (b) by comparing to SE (shown above distributions) 
  and RSE (below)
  using two-sided $t$-tests ($^*P \leq 0.05$).
  (\textbf{d}) 
  Cost distributions for each iteration displayed using box-and-whiskers plots as in~Figures~\ref{fig:exp1}e and~\ref{fig:exp2}d. 
  (\textbf{e,f}) One- and two-dimensional histograms of actions for different iterations ($k=0$ in (e), $k=9$ in (f)) with policies and game-theoretic equilibria overlaid (SE in (e), RSE in (f)).
  (\textbf{g}) 
  Error between measured and theoretically-predicted policy slopes at each iteration displayed as box-and-whiskers plots as in (b,c).
}
\end{figure}

In our third experiment (\figref{exp3}), the machine adapted its affine policy using a \emph{policy gradient} strategy~\cite{Chasnov2020-yo}. 
Trials 
again came in pairs, with the machine's policy in each pair 
differing this time only in the linear term. 
After a pair of trials,
the median costs of the trials were used to estimate the gradient of the machine's cost with respect to the linear term in its policy%
, and the linear term was adjusted in the direction opposing the gradient to decrease the cost.
Iterating this process shifted the distribution of median action vectors for a population of human subjects (distinct from the populations in the first two experiments) from the \emph{human-led Stackelberg equilibrium} (SE) toward the machine's \emph{global {\optimum}} (\subfigref{exp3}a), 
which can also be regarded as a \emph{reverse Stackelberg equilibrium}~\cite{Ho1982-eu} (RSE),
this time optimizing the machine's cost at the human's expense (\subfigref{exp3}d).
The shift we observed away from SE toward RSE from the first to last iterations was statistically significant in action space 
(\figref{exp3}b; $^*P \leq 0.05$; two-sided $t$-tests, df $19$; exact statistics in Table~\ref{tab:pvalues})
while the final policy distribution was significantly different from both SE and RSE policies 
(\figref{exp3}c; $^*P \leq 0.05$; two-sided $t$-tests, df $19$; exact statistics in Table~\ref{tab:pvalues}).
However, 
the machines' empirical policy gradients were not significantly different from theoretically-predicted values 
(\figref{exp3}g; $P > 0.05$; two-sided $t$-tests, df $19$; exact statistics in Table~\ref{tab:pvalues}),
and the final distribution of machine costs were not significantly different from the optimal value 
(\figref{exp3}d; $P > 0.05$; one-sided $t$-tests, df $19$; exact statistics in Table~\ref{tab:pvalues}),
suggesting that the machine can accurately estimate its policy gradient 
and 
minimize its cost.
In essence, the machine elevated its play by reasoning in the space of policies to steer the game outcome in this experiment to the point it desires in the joint action space.
We report results from variations of this experiment with different initializations and machine {\optima} in~Extended~Data~(Sections~\ref{sec:exp3-other-init},~\ref{sec:exp3-other-optima}).

\subsection*{Discussion}
When the machine played any policy in our experiments (i.e.\ when the machine's action $m$ was determined as a function of the human's action $h$), it effectively imposed a constraint on the human's optimization problem.
The policy could arise indirectly, as in the first experiment where the machine descended the gradient of its cost at a fast rate, or be employed directly, as in the second and third experiments.
In all three experiments, the empirical distributions of human actions or policies were consistent with the analytical solution of the human's constrained optimization problem for each machine policy 
(\subfigref{exp1}{d}; \subfigref{exp2}{b,c}; \subfigref{exp3}{b,c}).
This finding is significant because it shows that optimality of human behavior was robust with respect to the cost we prescribed and the constraints the machine imposed,
indicating our results may generalize to other settings where people (approximately) optimize their own utility function.
We report results from variations of all three experiments with non-quadratic cost functions in the Supplement (Section 
\ref{sec:exp-cobb}).

There is an exciting prospect for adaptive machines to assist humans in work and activities of daily living as
tele- or co-robots~\cite{nikolaidis2017game},
interfaces between computers and the brain or body~\cite{Perdikis2020-ze,De_Santis2021-aw},
and devices like exoskeletons or prosthetics~\cite{Felt2015-qc,Zhang2017-jq,slade2022personalizing}.
But designing adaptive algorithms that play well with humans -- who are constantly learning from and adapting to their world -- remains an open problem in robotics, neuroengineering, and machine learning~\cite{nikolaidis2017game,Recht2019-vf,Perdikis2020-ze}.
We 
validated game-theoretic
methods for machines to 
provide assistance by shaping  
outcomes during co-adaptive interactions with human partners.
Importantly, our methods do not entail solving an inverse optimization problem~\cite{li2019differential,Ng2000-aj} -- rather than estimating the human's cost function, our machines learn directly from human actions.
This feature may be valuable in the context of the emerging \emph{body-/human-in-the-loop optimization} paradigm for assistive devices~\cite{Felt2015-qc,Zhang2017-jq,slade2022personalizing}, where the machine's cost is deliberately chosen with deference to the human's metabolic energy consumption~\cite{Abram2022-rg} or other preferences~\cite{Ingraham2022-rj}.

Our results demonstrate the power of machines in co-adaptive interactions played with human opponents.
Although humans responded rationally at one level by choosing optimal actions in each experiment, the machine was able to ``outsmart'' its opponents over the course of the three experiments by playing higher-level games in the space of policies.
This machine advantage could be mitigated if the human rises to the same level of reasoning, but the machine could then go higher still, theoretically leading to a well-known infinite regress~\cite{Harsanyi1967}. 
We did not observe this regress in practice, possibly due to bounds on the computational resources available to our human subjects as well as our machines~\cite{gershman2015computational}.

\subsection*{Conclusion}
As machine algorithms permeate more aspects of daily life, it is important to understand the influence they can exert on humans to prevent undesirable behavior, ensure accountability, and maximize benefit to individuals and society~\cite{Thomas2019-qt, Cooper2022-wv}.
Although the capabilities of humans and machines alike are constrained by the resources available to them,
there are well-known limits on human rationality~\cite{Tversky1974-rp}
whereas machines benefit from sustained increases in computational resources, training data, and algorithmic innovation~\cite{Hilbert2011-ap, Jordan2015-tb}.
Here we showed that machines can unilaterally change their learning strategy to select from a wide range of theoretically-predicted outcomes in co-adaptation games played with human subjects.
Thus machine learning algorithms may have the power to aid human partners, for instance by supporting decision-making or providing assistance when someone's movement is impaired.
But when machine goals are misaligned with those of people, it may be necessary to impose limitations on algorithms to ensure the 
safety, 
autonomy,
and 
well-being
of people.

\bibliography{refs}
\clearpage

\newpage

\section*{Methods}

\paragraph{Experimental protocol.}
Human subjects were recruited using an online crowd-sourcing research platform 
\emph{Prolific}~\cite{Palan_Schitter_2018}.
Experiments were conducted using procedures approved by the University of Washington Institutional Review Board (UW IRB STUDY00013524).
Participant data were collected on a secure web server. 
Each experiment consisted of a sequence of 
trials:  14 
trials in the first experiment, 20 trials in the second and third experiments.
During each trial, participants used a web browser to view a graphical interface and provide manual input from a mouse or touchscreen to 
continually
determine the value of a scalar action $\actH\in\R$.  
This cursor input was scaled to the width of the participant's web browser window such that $h = -1$ corresponded to the left edge and $h = +1$ corresponded to the right edge.
Data were collected at 60 samples per second for a duration of 40 seconds per trial in the first experiment and 20 seconds per trial in the second and third experiments.
Human subjects were selected from the ``standard sample'' study distribution from all countries available on Prolific. 
Each subject participated in only one of the three experiments.
No other screening criteria were applied.

At the beginning of each experiment, an introduction screen was presented to participants with the task description and user instructions.
At the beginning of each trial, participants were instructed to move the cursor to a randomly-determined position.  
This procedure was used to introduce randomness in the experiment initialization and to assess participant attention. 
Throughout each trial, a rectangle's height displayed the current value of the human's cost $c_H(h,m)$ 
and participant was instructed to ``keep this [rectangle] as \emph{small} as possible'' 
by choosing an action $h\in\R$ while the machine updated its action $\actM\in\R$. 
A square root function was applied to cost values to make it easier for participants to perceive small differences in low cost values.
After a fixed duration, one trial ended and the next trial began. 
Participants were offered the opportunity to take a rest break for half a minute between every three trials.
The experiment ended after a fixed number of trials.
Afterward, the participant filled out a {task load survey~\cite{hart1988development}}
and optional feedback form. 
Each experiment lasted approximately 10--14 minutes and the participants received a fixed compensation of \$2 USD (all data was collected in 2020).
A video illustrating the first three trials of Experiment 1 is provided as Movie S1. 
The user interface presented to human subjects was identical in all experiments.  
However, the machine adapted its action and policy throughout each experiment, and the adaptation algorithm differed in each experiment.

\begin{figure}[ht]
\centering
\includegraphics[width=\textwidth]{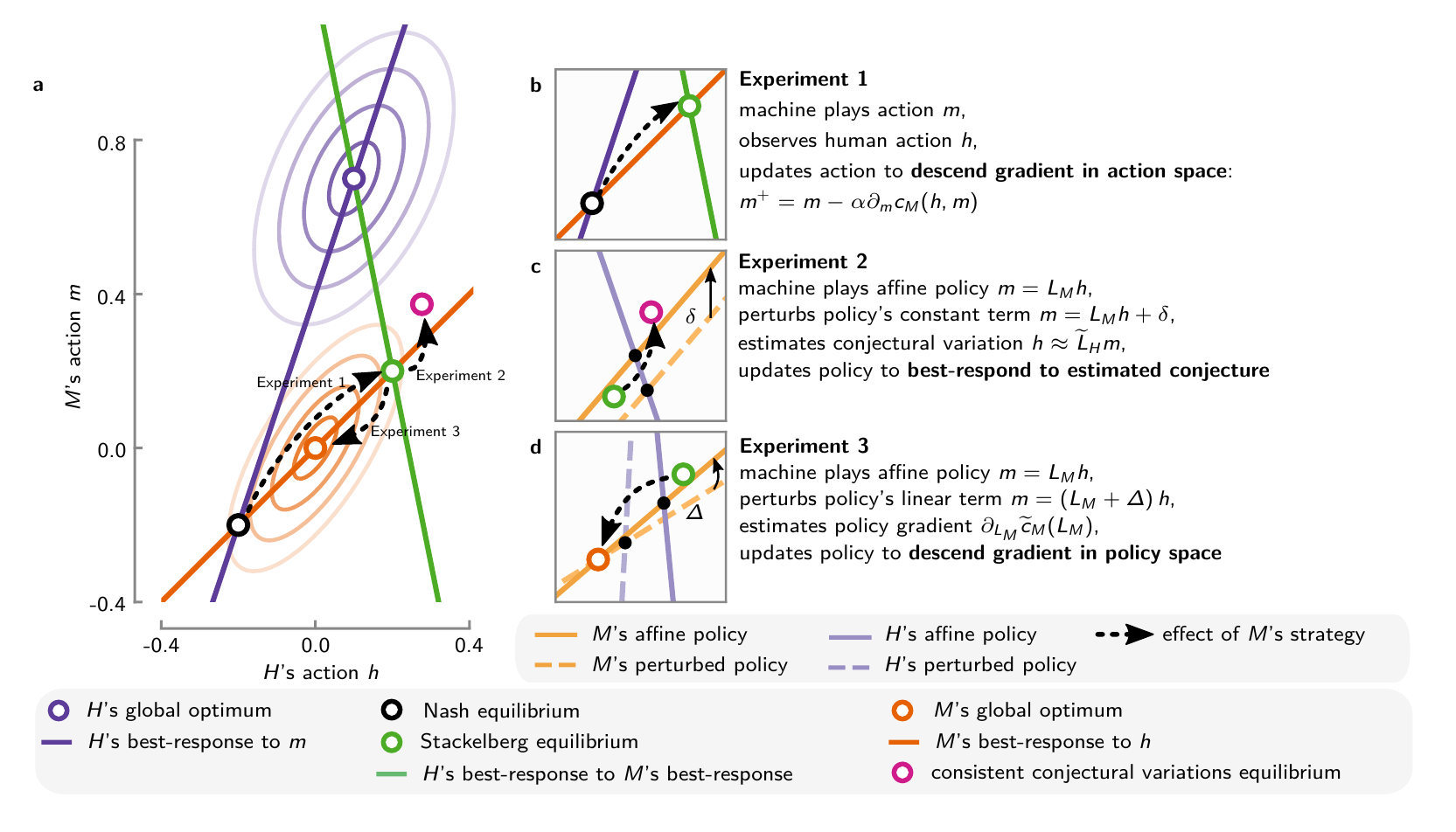}
  \caption[Co-adaptation game between human and machine]{\label{fig:overview}
  \scriptsize
  \textbf{Overview of co-adaptation experiment between human and machine.} 
 Human subject $H$ is instructed to provide manual input $\actH$ to make a black bar on a computer display as small as possible.
The machine $M$ has its own prescribed cost $\costM$ chosen to yield game-theoretic equilibria that are distinct from each other and from each player's global optima.
  (\textbf{a}) Joint action space illustrating game-theoretic equilibria and response functions determined from the costs prescribed to human and machine: \emph{global {\optima}} defined by minimizing with respect to both variables; \emph{best-response} functions defined by fixing one variable and minimizing with respect to the other. 
   Machine plays different strategies in three experiments:
   (\textbf{b}) gradient descent in \emph{Experiment 1};
   (\textbf{c})  conjectural variation in \emph{Experiment 2};
   (\textbf{d})  policy gradient descent in \emph{Experiment 3}.
  }
\end{figure}

\paragraph{Cost functions.}
In Experiments 1, 2, and 3,
participants were prescribed the quadratic cost function
\begin{linenomath}
\eqnn{costH}{
  \costH(\actH,\actM)
  &=\tfrac{1}{2}\actH^2 + \tfrac{7}{30}\actM^2 -\tfrac{1}{3}hm + \tfrac{2}{15}h-\tfrac{22}{75}m  + \tfrac{12}{125};
}
\end{linenomath}
the machine optimized the quadratic cost function
\begin{linenomath}
\eqnn{costM}{
  \costM(\actH,\actM)= \tfrac{1}{2}m^2+h^2-hm. 
}
These costs were designed such that the players' optima
and the constellation of relevant game-theoretic equilibria were distinct positions as listed in the Table~\ref{tab:equilibrium}.
During each trial of an experiment, the time series of actions from the trials were recorded as human actions $h_0,\dots,h_t,\dots,h_T$ and machine actions $m_0,\dots,m_t,\dots,m_T$, for a fixed number of samples $T$. 
At time $t$, the players experienced costs $c_H(h_t,m_t)$ and $c_M(h_t,m_t)$.
See Supplement~\secref{theory-def} for formal definitions of the relevant game-theoretic equilibria and 
 Supplement~\secref{derive-eq} for how the parameters for the costs were chosen.
\end{linenomath}

\paragraph{Experiment 1: gradient descent in action space.}
In the first experiment, the machine adapted its action using gradient descent,
\begin{linenomath}
\eqnn{exp1-gd}{
\actM^+ &= \actM-\alpha\, \partial_\actM \costM(\actH,\actM),
}
\end{linenomath}
with one of seven different choices of {\learningrate} $\alpha \in\set{0, 0.003, 0.01, 0.03, 0.1, 0.3, 1}$.
At the slowest adaptation rate $\alpha = 0$, the  machine implemented the constant policy $\actM = \actMNash$, which is the machine's component of the game's Nash equilibrium.
At the fastest adaptation rate $\alpha = 1$, the gradient descent iterations in~\eqref{eq:exp1-gd} are such that the machine implements the linear policy $\actM = \actH$. 
Each condition was experienced twice by each human subject, once per symmetry (described in the next paragraph), in randomized order.

To help prevent human subjects from memorizing the location of game equilibria, 
at the beginning of each trial 
a variable $s$ was chosen uniformly at random from $\set{-1,+1}$ and 
the map $\actH \mapsto s\, \actH$  was applied to the human subject's manual input for the duration of the trial. When the variable's value was $s = -1$, this had the effect of applying a ``mirror'' symmetry to the input.
The joint action was initialized uniformly at random in the square $[-0.4,+0.4] \times [-0.4,+0.4]\subset\R^2$. 
Each trial lasted 40 seconds.

\paragraph{Experiment 2: conjectural variation in policy space.}
In the second experiment, the machine adapted its policy by estimating a \emph{conjecture} about the human's \emph{policy}.
To collect the data that was used to form its estimate, the machine played an affine policy in two consecutive trials that differed solely in the constant term,
\eqnns{exp2:actM}{
\text{nominal policy}\ 
\actM  
\label{eq:exp2:actM:nom}
& = \ConjM \actH,\\
\text{perturbed policy}\ 
\actM^\prime  
\label{eq:exp2:actM:pert}
& = \ConjM \actH^\prime 
+ \delta.
}
The machine used the median action vectors $(\est{h},\est{m})$, $(\est{h}',\est{m}')$ from the pair of trials to estimate a conjecture about the human's policy using a ratio of differences, 
\begin{linenomath}
\eqnn{exp2:estConjH}{
\est{\Conj}_H &= \frac{\est{\actH}^{\prime}-\est{\actH}}{\est{\actM}^\prime-\est{\actM}},
}
\end{linenomath}
which is shown to be an 
 estimate of
the variation of the human's action in response to machine action
in Proposition~\ref{prop:exp2-cv} of Supplement~\secref{scalar-exp2}.
The machine used this estimate of the human's policy to update its policy as 
\begin{linenomath}
\eqnns{exp2:ConjM}{
L_M^+=\frac{1-2\est{L}_H}{1-\est{L}_H},
}
\end{linenomath}
which is shown to be the machine's best-response given its conjecture about the human's policy in Supplement~\secref{theory-scalar}.
In the next pair of trials, the machine employs $m=L_M^+ h+\ell_M^+$ as its policy.
This conjectural variation process was iterated 10 times starting from the initial conjecture $\est{L}_H = 0$, which yields the initial best-response policy $m = h$.

In this experiment, the machine's policy slopes $L_{M,0},L_{M,1},\dots,L_{M,k},\dots,L_{M,K-1}$ and the machine's conjectures about the human's policy slopes $\est{L}_{H,0},\est{L}_{H,1},\dots,\est{L}_{H,k},\dots,\est{L}_{H,K-1}$ were recorded for each conjectural variation iteration $k\in\set{0,\dots,K-1}$ where $K = 10\,\text{iterations}$. 
In addition, the time series of actions within each trial as in the first experiment, with each trial now lasting only 20 seconds, yielding $T=1200\,\text{samples}$ used to compute the median action vectors used in~\eqref{eq:exp2:estConjH}.

\paragraph{Experiment 3: gradient descent in policy space.}
In the third experiment, the machine adapted its policy using a policy gradient strategy by playing an affine policy in two consecutive trials that differed only in the linear term,%
\begin{linenomath}
\eqnns{exp3:actM}{%
\text{nominal policy}\ 
\actM  
\label{eq:exp3:actM:nom}
& = \ConjM \actH,\\
\text{perturbed policy}\ 
\actM^\prime  
\label{eq:exp3:actM:pert}
& = (\ConjM + \Delta) \actH^\prime.
}
\end{linenomath}
The machine used the median action vectors $(\est{h},\est{m})$, $(\est{h}',\est{m}')$ from the pair of trials to estimate the gradient of the machine's cost with respect to the linear term in its policy, 
and this linear term was adjusted to decrease the cost.
Specifically, an auxiliary cost was defined as
\begin{linenomath}
\eqnn{exp3:costM}{
\widetilde{\costM}(L_M):=\costM\paren{h,\ConjM (h-h_M^\ast)+m_M^\ast},
}
\end{linenomath}
and the pair of trials were used to obtain a finite-difference estimate of the gradient of the machine's cost with respect to the slope of the machine's policy, 
\begin{linenomath}
\eqnn{exp3:ConjH}{
&\partial_{L_M} \est{\costM}(L_M)
\approx \frac{1}{\Delta}\big(\widetilde{\costM}(L_M+\Delta)  - \widetilde{\costM}(L_M)
\big).
}
\end{linenomath}
 The machine used this derivative estimate to update the linear term in its policy by descending its cost gradient, 
\begin{linenomath}
\eqnn{exp3:ConjM}{ 
\ConjM^+ &= \ConjM - \gamma\, \partial_{\ConjM}\est{\costM}(L_M)\\ 
}
\end{linenomath}
where $\gamma$ is the policy gradient {\learningrate} parameter ($\gamma=2$ in this Experiment).
\paragraph{Statistical analyses.}
To determine the statistical significance of our results, we use one- or two-sided $t$-tests with threshold $P \leq 0.05$ applied to distributions of median data from populations of $n=20$ subjects. To estimate the effect size, we calculated Cohen's $d$ by subtracting the equilibrium value from the mean of the distribution then dividing that by the standard deviation of the distribution.

\begin{table}[t]
{\scriptsize 
\renewcommand{\arraystretch}{1.8}
\arrayrulecolor{gray}
\centering
{
\footnotesize
\vspace{5em} Cost functions and game-theoretic equilibria}
\begin{tabular*}{\linewidth}{|l|l@{\extracolsep{\fill}}l|}
\hline
\bf $H$'s cost function & \bf $M$'s cost function & \\
$c_H(h,m)=\tfrac{1}{2}\actH^2 + \tfrac{7}{30}\actM^2 -\tfrac{1}{3}hm + \tfrac{2}{15}h-\tfrac{22}{75}m  + \tfrac{12}{125}\qquad $  & $c_M(h,m)=\tfrac{1}{2}m^2+h^2-hm\qquad$ & \\
\hline
\end{tabular*}

\vspace{1ex}

\begin{tabular*}{\linewidth}{|l@{\extracolsep{\fill}}l@{\extracolsep{\fill}}l|}
\hline
   \bf  game-theoretic equilibria & \bf $H$'s and $M$'s actions & \bf $H$'s and $M$'s policy slopes \\
	$H$'s optimum & $(\Opt{\actH}_H,\Opt{\actM}_H) = (+0.1,+0.7)$ & \\
	$M$'s optimum & $(\Opt{\actH}_M,\Opt{\actM}_M) = (0,0)$ & \\
	Nash equilibrium & $(\Nash{\actH},\Nash{\actM}) = (-0.2,-0.2)$ & \\
	human-led Stackelberg equilibrium & $(\Stack{\actH},\Stack{\actM}) = (+0.2,+0.2)$ & $\Stack{L_H}=-0.2,\qquad\, \Stack{L_M}=1$\\
	consistent conjectural variations equilibrium & $(\Ccve{\actH},\Ccve{\actM}) \approx (0.276,0.373)$ & $\Ccve{L_H}\approx -0.54,\ \Ccve{L_M}\approx+1.35$ \\
	machine-led reverse Stackelberg equilibrium  & $(\RStack{\actH},\RStack{\actM})=(0,0)$ 
 & $\RStack{L_H}=1/7,\qquad\, \RStack{L_M}= 5/11$\\
 \ \ (equal to $M$'s optimum)&&\\
 \hline
\end{tabular*}
\caption[Cost functions and game-equilibria]{%
\textbf{Cost functions and game-theoretic equilibria of the game studied in Experiments 1, 2, and 3.} 
The Supplement details how the costs were chosen: 
Section~\ref{sec:derive-eq} describes 
the general approach, 
and
Section~\ref{subsec:choosing-parameters} specializes to the game studied here.
}
\label{tab:equilibrium}}

\end{table}

\subsection*{Data availability}
All data are publicly available in a Code Ocean capsule, {\tt codeocean.com/capsule/6975866}. 

\subsection*{Code availability}
The data and analysis scripts needed to reproduce all figures and statistical results reported in both the main paper and supplement are publicly available in a Code Ocean capsule, {\tt codeocean.com/ capsule/6975866}. 
The sourcecode used to conduct experiments on the Prolific platform are publicly available on GitHub, {\tt github.com/dynams/web}.


\subsection*{Acknowledgements}

This work was funded by the National Science Foundation (Awards \#1836819, \#2045014).
Benjamin Chasnov was funded in part by
Computational Neuroscience Graduate Training Program (NIH 5T90DA032436-09 MPI).

\subsection*{Author contributions}
B.J.C., 
L.J.R., 
and
S.A.B.\  were responsible for methodology design and manuscript preparation;
B.J.C.\ collected and analyzed experimental data and prepared figures.

\subsection*{Competing interests}
The authors declare no competing interests.

\subsection*{Additional information}
Supplementary Information is available for this paper.
Correspondence and requests for materials should be addressed to S.A.B. (E-mail:  {\tt sburden@uw.edu}).
\newpage

\clearpage

\renewcommand{\familydefault}{\sfdefault}

\clearpage
\newpage

\renewcommand{\thesection}{S\arabic{section}}
\renewcommand{\thesubsection}{S\arabic{section}.\arabic{subsection}}
\renewcommand{\thefigure}{S\arabic{figure}}
\renewcommand{\thetable}{S\arabic{table}}
\setcounter{figure}{0}
\setcounter{table}{0}

\title{Supplementary Information for \\
\emph{Human adaptation to adaptive machines \\
converges to game-theoretic equilibria}}

\maketitle

\vspace{-3em}

\noindent  
{\bf List of supplementary materials:}

Section S1 -- S7

Figure S1 -- S6

Table S1 -- S4

Protocol S1 -- S3

Sourcecode S0 -- S3

Movie S1

\section*{Summary of supplementary materials }
This Supplementary Information supports the claims in the main paper.

The formal mathematical 
definitions of the game-theoretic equilibrium solutions are in
\secref{theory-def}.
The parameters of a pair of quadratic costs are determined by the equilibrium solutions in
\secref{derive-eq}.
The analysis of the game from the main paper is provided in \secref{theory-scalar}. 
Experiments 1 on gradient descent in action space is analyzed in 
\secref{scalar-exp1}. 
Experiment 2 on conjectural variations in policy space is analyzed in
\secref{scalar-exp2}.
Experiment 3 on gradient descent in policy space is analyzed in 
\secref{scalar-exp3}.

Interpretations of the conjectural iteration are provided in
\secref{interpretation}. 
The related economic idea of comparative statics is described in
\secref{interpretation-comparative-statics} 
and Taylor approximation is used to characterize consistent conjectures in 
\secref{interpretation-basar}.

\newpage

\renewcommand{\contentsname}{Sections}
\renewcommand{\listfigurename}{Figures}
\renewcommand{\listtablename}{Tables}

\tableofcontents
\newpage
\listoffigures
\listoftables

\section*{Protocols}

\hspace{\parindent}S1\qquad  Algorithm description of Experiment 1 \dotfill\ \pageref{protocol:exp1}

S2\qquad Algorithm description of Experiment 2 \dotfill\ \pageref{protocol:exp2}

S3\qquad Algorithm description of Experiment 3 \dotfill\ \pageref{protocol:exp2}

\section*{Sourcecodes}

\hspace{\parindent}S0\qquad  Definitions of parameters, cost functions and gradients of two players \dotfill \pageref{sourcecode:exp0}

S1\qquad  Numerical simulation of Experiment 1 \dotfill\ \pageref{sourcecode:exp1}

S2\qquad Numerical simulation of Experiment 2 \dotfill\ \pageref{sourcecode:exp2} 

S3\qquad Numerical simulation of Experiment 3 \dotfill\ \pageref{sourcecode:exp3}

\newpage

\section{Game theory definitions}\label{sec:theory-def}

We model co-adaptation between humans and machines using game theory~\cite{von1947theory,bacsar1998dynamic}.
In this model, the human $H$ chooses action $\actH\in\ActH $ while the machine $M$ chooses action $\actM\in\ActM$ to minimize their respective \emph{cost functions} $\costH, \costM : \ActH\times\ActM \into\R$,
\eqnns{game}{
\min_{\actH}\ &\costH(\actH,\actM),\\ 
\min_{\actM}\ &\costM(\actH,\actM).
}
It is important to note that the optimization problems in~\eqref{eq:game} are coupled. 
Since both problems must be considered simultaneously, there is no obvious candidate for a ``solution'' concept (in contrast to the case of pure optimization problems, where (local) minimizers of the single cost function are the obvious goals).
Thus, we designed experiments to study a variety of candidate solution concepts that arise naturally in different contexts.
We demonstrate that Nash, Stackelberg, consistent conjectural variations equilibria, and players' global optima are possible outcomes of the experiments.

\subsection{Nash and Stackelberg equilibria}
In games with simultaneous play where players do not form conjectures about the others' policy, a natural candidate solution concept is the \emph{Nash equilibrium}~\citedef{4.1}{bacsar1998dynamic}.

\paragraph*{Definition:} The joint action 
$(\Nash{\actH},\Nash{\actM})\in\ActH\times\ActM$ 
constitutes a \emph{Nash equilibrium} (\NE) if 
\eqnns{NE}{
\Nash{\actH} &= \arg\min_{\actH} \costH(\actH,\Nash{\actM}),\\
\Nash{\actM} &= \arg\min_{\actM} \costM(\Nash{\actH},\actM). 
}

In games with ordered play where the \emph{leader} (e.g.\ human) has knowledge of how the \emph{follower} (e.g.\ machine) responds to choosing its own action,  a natural candidate solution concept is the \emph{(human-led) Stackelberg equilibrium}~\citedef{4.6}{bacsar1998dynamic}.
\paragraph*{Definition:} The joint action 
$(\Stack{\actH},\Stack{\actM})\in\ActH\times\ActM$ 
constitutes a \emph{(human-led) Stackelberg equilibrium} (\SE) if 
\eqnns{SE}{
\Stack{\actH} &= \arg\min_{\actH} \set{\costH\paren{\actH,\actM} \st \actM = \arg\min_{\actM'} \costM(\actH,\actM')},\\
\Stack{\actM} &= \arg\min_{\actM}\  \costM(\Stack{\actH},\actM).
}

\noindent
The Stackelberg equilibrium is a solution concept  
that arises when  one player (the leader) anticipates or models another player's (the follower's) best response.

\subsection{Consistent conjectural variations equilibria}

In repeated games where each player gets to observe the other's actions and policies, players may develop internal models or {conjectures} for how they expect the other to play.
A natural candidate solution concept in this case is the \emph{consistent conjectural variations equilibrium}~\citedef{4.9}{bacsar1998dynamic}.

For a given pair%
\footnote{We use the shorthand $\set{A \into B}$  to denote the set of functions from $A$ to $B$.}
$(\Ccve{\Hconj},\Ccve{\Mconj})\in\set{\ActM\to\ActH}\times\set{\ActH\to\ActM}$, 
denote the unique fixed points $(\Ccve{h},\Ccve{m})\in\ActH\times\ActM$ satisfying
\eqnns{ccve-fp-1}{
\Ccve{h} &= \Ccve{\Hconj}\circ\Ccve{\Mconj}(\Ccve{h}),\\
\Ccve{m} &= \Ccve{\Mconj}\circ\Ccve{\Hconj}(\Ccve{m}).
}
Let 
\eqnns{ccve-def-1}{
\Delta \Ccve{v}_H(m)&=\Ccve{\Hconj}(m)-\Ccve{\Hconj}(\Ccve{m}),\\
\Delta \Ccve{v}_M(h)&=\Ccve{\Mconj}(h)-\Ccve{\Mconj}(\Ccve{h}),
}
be the differential reactions of each player under their policies $(\Ccve{\Hconj},\Ccve{\Mconj})$ to a deviation from the joint action $(\Ccve{h},\Ccve{m})$ to $(m,h)$.

\paragraph*{Definition:} The joint action 
$(\Ccve{\actH},\Ccve{\actM})\in\ActH\times\ActM$ 
together with the conjectures 
$\Ccve\Mconj:\ActH\into\ActM$, $\Ccve\Hconj:\ActM\into\ActH$ 
constitute a \emph{consistent conjectural variations equilibrium} (\CCVE) if 
we have the consistency of actions 
\eqn{
\Ccve{h}&=\arg\min_{h}\set{\costH(h,m) \mid m= \Ccve{\Mconj}(h)},\\
\Ccve{m}&=\arg\min_{m}\set{\costM(h,m)\mid h =\Ccve{\Hconj}(m)},
}
and consistency of policies 
\eqn{
\Ccve{\Hconj}(m)&=\arg\min_{h}\ c_H(h,m+\Delta \Ccve{v}_M(h)),\\
\Ccve{\Mconj}(h)&=\arg\min_{m}\ c_M(h+\Delta \Ccve{v}_H(m),m).
}

\noindent
The consistent conjectural variations equilibrium is a solution concept  
that arises when  players anticipate each other's actions and reactions. 

\subsection{Reverse Stackelberg equilibria}

In games where one player (the leader) has the ability to impose a policy before the other player (the follower) who responds to the policy, the candidate solution concept for this case is the \emph{reverse Stackelberg equilibrium} \cite{Ho1981-wu,Ho1982-eu}. 
The machine acts as the leader in this game, and announces policy is $\pi:\ActH\to\ActM$. Assume the human's best response to machine policy $\pi$ is
$r:(\ActH\to \ActM)\to \ActH$ given by a constrained optimization problem:
\eqn{ 
r(\pi):=\arg\min_{h}\left\{c_H(h,m)\mid m = \pi(h)\right\}.
}

\paragraph*{Definition:} 
The joint action $(\RStack{h}, \RStack{m})\in\ActH\times\ActM$   together with machine policy $\RStack{\pi}:\mc H\to \mc M$ constitute a \emph{reverse Stackelberg equilibrium} (RSE) if
\eqnns{}{
\RStack{\pi}&=\arg\min_{\pi}\left\{c_H(h,m )\mid m = \pi(h),\ h=r(\pi) )\right\},\\ 
\RStack{h}&=r(\RStack{\pi}),\\
\RStack{m}&=\RStack{\pi}(\RStack{h}).
}

\noindent
If the reverse Stackelberg problem is incentive-controllable~\cite{Ho1981-wu}, then the reverse Stackelberg equilibrium is the machine's global optimum.

\section{Game design}\label{sec:derive-eq}
In this section, the equilibrium points are derived 
by solving linear equations while enforcing certain second-order and stability conditions. The general quadratic costs 
are given by
\eqnns{quadcosts}{
  \costH(\actH,\actM)&= \tfrac{1}{2}\actH^\top \HA \actH + \actH^\top \HB \actM + \tfrac{1}{2}\actM^\top \HD \actM  + \Hb^\top \actH + \Hd^\top\actM +a_H,\label{eq:quadcosts1}\\
\costM(\actH,\actM)&={\tfrac{1}{2}\actM^\top \MA \actM  + {\actM^\top \MB \actH} + \tfrac{1}{2}\actH^\top \MD \actH} +  {\Mb^\top \actM + \Md^\top\actH} +a_M.\label{eq:quadcosts2}
}
where 
actions $h\in\R^{p},\ m\in\R^{q}$ are vectors with $p\geq 1$ and $q\geq 1$, 
cost parameters $A_H\in\R^{p\times p},\ D_H\in\R^{q\times q},\ A_M\in\R^{q\times q},\ D_M \in \R^{p\times p}$ are symmetric matrices,
 $B_H\in\R^{p\times q},\ B_M \in \R^{q\times p}$ are matrices, $b_H\in\R^{p},\ d_H\in\R^{q},\ b_M\in\R^{p},\ d_M\in\R^{q}$ are vectors and $a_H\in\R,\ a_M\in\R$ are scalars.

The cost parameters are chosen so that 
the equilibrium points are located at chosen points in the action spaces. 
Without loss of generality, $A_H$ and $A_M$ are the identity matrices to set the (arbitrary) scale for each player's cost. 
Subsequently, $a_H,a_M$ are determined such that the minimum cost values for both players are 0. 
Finally, and also without loss of generality, $b_M=d_M=0$ is determined to center the machine's cost at the origin in the joint action space. 
The six coefficients that remain to be determined are $B_H,B_M,D_H,D_M,b_H,d_H$. The parameters will determine the location of the equilibrium solutions of the game.

In the main paper, the action spaces are scalar, i.e. $p=q=1$. The parameters were chosen to be $A_H=1,\ B_H=-1/3,\ D_H=7/15,\ b_H=2/15,\ d_H=-22/75$ for the human and $A_M=1,\ B_M=-1,\ D_M=2,\ b_M=0,\ d_M=0$ for the machine. The players' optima for this game are 
\eqn{
 (\Opt{h}_H,\Opt{m}_H)&=(0.1,0.7),\\
  (\Opt{h}_M,\Opt{m}_M)&=(0,0),
 }
  and the game-theoretic equilibria are 
 \eqn{
 (\Nash{h},\Nash{m})&=(-0.2,-0.2),\\
 (\Stack{h},\Stack{m})&=(0.2,0.2),\\
 (\Ccve{h},\Ccve{m})&\approx (0.276,0.373),\\
 (\RStack{h},\RStack{m})&= (0,0).
 }
In the follwoing subsections, the first and second order conditions for the solutions of optimization problems are written out for the costs $c_H, c_M$ in \eqref{eq:quadcosts1} and \eqref{eq:quadcosts2}.
\subsection{Global optima}
The global optimization problems for the two players are
\eqn{
(\Opt{h}_H,\Opt{m}_H)&=\amin_{\actH,\actM}\ c_H(h,m),\\
(\Opt{h}_M,\Opt{m}_M)&=\amin_{\actH,\actM}\ c_M(h,m)
}
which have first-order  
conditions
\eqn{
\begin{bmatrix}
A_H&B_H \\
B_H^\top & D_H
\end{bmatrix}
\begin{bmatrix}
\Opt{\actH}_H\\
\Opt{\actM}_H
\end{bmatrix}
+
\begin{bmatrix}
b_H\\
d_H
\end{bmatrix}
=0
\text{ and }
\begin{bmatrix}
D_M & B_M^\top \\ 
B_M & A_M
\end{bmatrix}
\begin{bmatrix}
\Opt{\actH}_M\\
\Opt{\actM}_M
\end{bmatrix}
+
\begin{bmatrix}
d_M\\
b_M
\end{bmatrix}
=0,
}
and second-order  
conditions that
 $\begin{bmatrix}
A_H&B_H \\
B_H^\top & D_H
\end{bmatrix}$ and 
$\begin{bmatrix}
D_M & B_M^\top \\ 
B_M & A_M
\end{bmatrix}$ are  positive semi-definite. 
 See Proposition 1.1.1 in~\cite{Bertsekas1999-an} for the formal statement of these conditions.
\subsection{Nash equilibrium}
The coupled optimization problems for a Nash equilibrium $(\Nash{\actH},\Nash{\actM})$ are
\eqn{
\Nash{h}&=\amin_{h}\ c_H(h,\Nash{m}),\\
\Nash{m}&=\amin_m\ c_M(\Nash{h},m),
}
which
have first-order 
conditions 
\begin{align*}
	\begin{bmatrix}A_H&B_H\\
	B_M&A_M\end{bmatrix}
	\begin{bmatrix}
	\Nash{\actH}\\
	\Nash{\actM}
	\end{bmatrix}
	+
	\begin{bmatrix}b_H\\b_M\end{bmatrix}=0
\end{align*}
and second-order 
conditions $A_H\geq 0$ and $A_M\geq 0$. 
If the Jacobian $
\begin{bmatrix}
A_H & B_H\\ B_M & A_M
\end{bmatrix}$
has eigenvalues with positive real parts, then the Nash equilibrium is stable under gradient play. 

See Proposition 1 in \cite{ratliff2016characterization} for necessary conditions for a local Nash equilibrium and for the stability result for continuous-time gradient play dynamics  $\dot h=-\partial_hc_H(h,m),\ \dot m=-\partial_mc_M(h,m)$. See Proposition 2 in \cite{Chasnov2020-yo} for the corresponding discrete-time gradient play dynamics $h^+=h-\beta \partial_hc_H(h,m),\ m^+=m-\alpha\partial_Mc_M(h,m)$ for learning rates $\alpha,\beta>0$ and learning rate ratio $\tau=\alpha/\beta$. As the learning rate ratio $\tau$ tends to $\infty$, the machine's action $m$ adapts at a faster rate than $h$, which imposes a timescale separation between the two players. 

\subsection{Human-led Stackelberg equilibrium}
The coupled optimization problems for a human-led Stackelberg equilibrium $(\Stack{h},\Stack{m})$ are
\eqn{
\Stack{h}&=\amin_h\ \left\{ c_H(h,m')|\ m'=\amin_m\ c_H(h,m)\right\},\\
\Stack{m}&=\amin_m\ c_M(\Stack{h},m),
}
which have first-order conditions
\begin{align*}
\begin{bmatrix}	
A_H + L_{M,0}^\top  B_H^\top & B_H +L_{M,0}^\top D_H	 \\
B_M & A_M
\end{bmatrix}
\begin{bmatrix}
\Stack{\actH}\\
\Stack{\actM}
\end{bmatrix}+
\begin{bmatrix}
b_H+L_{M,0}^\top d_H
\\
b_M
\end{bmatrix}
=0
\end{align*}
with  $L_{M,0}=-A_M^{-1}B_M$, and  second-order conditions $A_M>0$, $A_H-B_H A_M^{-1} B_M > 0$.
See Proposition 4.3 in \cite{bacsar1998dynamic} for a quadratic game formulation of the Stackelberg equilibrium, which admits only a pure-strategy Stackelberg equilibrium.
 See Proposition 1 in \cite{fiez2020implicit} for conditions for a local Stackelberg equilibrium. 
\subsection{k-level conjectural variations equilibrium}
\label{sec:supp:k-level}
The coupled optimization problems for an intermediate conjectural variations equilibrium where  the human maintains a consistent conjecture of the machine are
\eqn{
\Cve{h}_{k+1}&=\amin_h\ \big\{ c_H(h,m')|\ m'=L_{M,k}(h-\Opt{h}_M)+\Opt{m}_M\big\}, \\
\Cve{m}_{k}&=\amin_m\ \big\{ c_M(h',m)|\ h'=L_{H,k-1}(m-\Opt{m}_H)+\Opt{h}_H\big\}, 
}
which have first-order optimality conditions
\begin{align*}
\begin{bmatrix}	
A_H + L_{M,k}^\top  B_H^\top & B_H +L_{M,k}^\top D_H	 \\
B_M + L_{H,k-1}^\top D_M & A_M + L_{H,k-1}^\top B_M^\top
\end{bmatrix}
\begin{bmatrix}
\Cve{h}_{k+1}\\
\Cve{m}_k
\end{bmatrix}+
\begin{bmatrix}
b_H+L_{M,k}^\top d_H\\
b_M+L_{H,k-1}^\top d_M
\end{bmatrix}=0	
\end{align*}
with initial condition $L_{M.0}=-A_M^{-1}B_M$ and iteration
\begin{align*}
L_{H,k+1}&=-(A_H+L_{M,k}^\top B_H^\top)^{-1}({B_H+L_{M,k}^\top D_H})
\\	
L_{M,k}&=-(A_M+L_{H,k-1}^\top B_M^\top)^{-1}({B_M+L_{H,k-1}^\top D_M})
\end{align*}
for $k=0,1,2,\dots$  with and the assumption that  $A_H+ B_HL_{M,k}$ and $A_M+ B_M L_{H,k-1}$ are invertible.
 See~\secref{theory-scalar} for more information about conditions under which this iteration converges for the particular parameters of the costs used in the main experiments. 

\subsection{Consistent conjectural variations equilibrium}
From \citedef{4.9}{bacsar1998dynamic},
the coupled optimization problems for the consistent conjectural variation equilibria are 
\eqn{
\Ccve{h}&=\amin_h\ \big\{ c_H(h,m')\mid m'= \Ccve{L}_{M}(h-\Opt{h}_M)+\Opt{m}_M \big\} \\
\Ccve{m}&=\amin_m\ \big\{ c_M(h',m)\mid h'= \Ccve{L}_{H}(m-\Opt{m}_H)+\Opt{h}_M \big\} 
}
where $\Ccve{L}_{M},\Ccve{L}_{H}$ solves the optimality conditions in the policy space
 equations from \citedef{4.10}{bacsar1998dynamic}:
\begin{align*}
A_M\Ccve{L}_{M}+{\Ccve{L}_{H}}^\top B_M^\top \Ccve{L}_{M}+{\Ccve{L}_{H}}^\top D_M+B_M&=0,\\
A_H\Ccve{L}_{H}+{\Ccve{L}_{M}}^\top B_H^\top\Ccve{L}_{H} +{\Ccve{L}_{M}}^\top D_H+B_H&=0.
\end{align*}
The first-order optimality conditions in the action space of the coupled optimization problems are
\begin{align*}
\begin{bmatrix}	
A_H + {\Ccve{L}_{M}}^\top B_H^\top & B_H +{\Ccve{L}_{M}}^\top D_H	 \\
B_M + {\Ccve{L}_{H}}^\top D_M & A_M + {\Ccve{L}_{H}}^\top B_M 
\end{bmatrix}
\begin{bmatrix}
\Ccve{h} \\
\Ccve{m}
\end{bmatrix}+
\begin{bmatrix}
b_H+{\Ccve{L}_{M}}^\top d_H\\
b_M+{\Ccve{L}_{H}}^\top d_M
\end{bmatrix}=0.
\end{align*}
Proposition 4.5 in~\cite{bacsar1998dynamic} states that if a game admits a unique Nash equilibirum, then the Nash equilibrium is also a CCVE with the Nash actions as constant policies.

\subsection{Machine-led reverse Stackelberg equilibrium} 
The coupled optimization problems corresponding to a machine-led reverse Stackelberg equilibrium  are given by:
\eqn{
\RStack{r}_H(L_M)&=\amin_h\ \big\{ c_H(h,m')\mid m'={L}_M(h-h_M^*)+m_M^*\big\}\\
\RStack{L}_{M}&=\amin_{L_M}\ \big\{ c_M(\RStack{r}_H(L_M),m')\mid m'={L}_M (\RStack{r}_H(L_M)-h_M^*)+m_M^*)\big\}
}
where the human forms a consistent conjecture of the machine, and the machine assumes that the human responds optimally to the machine's policy slope.
The reverse Stackelberg equilibrium is $(\RStack{h}, \RStack{m})$, which by the \cite{basar1979closed,groot2013thesis}, satisfies the same conditions that the machine's optimum satisfies, i.e. 
\begin{align*}
\begin{bmatrix}
A_M & B_M \\ 
B_M^\top & D_M
\end{bmatrix}
\begin{bmatrix}
\RStack{h}\\ 
\RStack{m}
\end{bmatrix}
+
\begin{bmatrix}
b_M\\
d_M
\end{bmatrix}
=0
\end{align*}
as well as first-order optimality conditions
\begin{align*}
\begin{bmatrix}	
A_H + {\RStack{L}_{M}}^\top  B_H^\top & B_M +{\RStack{L}_{M}}^\top D_H	 \\
-{\RStack{L}_{M}}  & I
\end{bmatrix}
\begin{bmatrix}
\RStack{h} \\
\RStack{m}
\end{bmatrix}+
\begin{bmatrix}
b_H+{\RStack{L}_{M}}^\top d_H\\
m_M^*-{\RStack{L}_{M}}^\top h_M^*
\end{bmatrix}=0
\end{align*}
where we need to also guarantee that the Jacobian is stable.
The second-order condition is $A_H + B_H\RStack{L}_{M}> 0$.
See Section III.B in  \cite{Ho1981-wu} for a method to solve reverse Stackelberg problems, relying on the property of linear incentive controllability. See \cite{groot2013thesis} for an overview of results and the computation of optimal policies. See Proposition 1 of \cite{zheng1982existence} for existence of optimal affine leader policies.

\subsection{Choosing parameters for a two-player game with single-dimensional actions} 
\label{subsec:choosing-parameters}

Given quadratic costs with scalar actions $h\in\R,\ m\in\R$,
\eqn{
  \costH(\actH,\actM)&= \tfrac{1}{2} \HA \actH^2 \ \,+ \HB \actH \actM + \tfrac{1}{2}\HD \actM^2  + \Hb \actH + \Hd \actM +a_H,\\
\costM(\actH,\actM)&={\tfrac{1}{2} \MA \actM^2  + {\MB \actH\actM} + \tfrac{1}{2}\MD \actH^2} +  {\Mb \actM + \Md\actH} +a_M.
}
Without loss of generality, $A_H=1$ and $A_M=1$ to set the scale for each player's cost.
The parameters expressed in terms of the optima $(\Opt{h}_H,\Opt{m}_H)$ and $(\Opt{h}_M,\Opt{m}_M)$ are
\eqn{
a_H &= \tfrac{1}{2}A_H {h_H^\ast} ^2 + B_H h_H^\ast m_H^\ast + \tfrac{1}{2}D_H {m_H^\ast}^2,
& b_H & = -A_Hh_H^\ast-B_Hm_H^\ast,
&d_H & = -B_Hh_H^\ast-D_Hm_H^\ast,\\
a_M &= \tfrac{1}{2} A_M {m_M^\ast} ^2 + B_M h_M^\ast m_M^\ast + \tfrac{1}{2}D_M {h_M^\ast}^2,
& b_M & = -A_Mm_M^\ast-B_Mh_M^\ast,
\ & 
d_M & = -B_Mm_M^\ast-D_Mh_M^\ast.\quad
}
The parameters expressed in terms of the optima and the Nash equilibrium $(\Nash{h},\Nash{m})$ are
\eqn{
B_H = -\frac{h_H^\ast - \Nash{h}}{m_H^\ast-\Nash{m}},\
B_M = -\frac{m_M^\ast - \Nash{m}}{h_M^\ast-\Nash{h}}. \\
}
The parameter expressed in terms of the optima and the human-led Stackelberg equilibrium $(\Stack{h},\Stack{m})$ is
\begin{align*}
D_H &= \frac{B_H\big(h_M^\ast m_H^\ast + h_H^\ast m_M^\ast - (m_H^\ast + m_M^\ast-\Stack{m}) \Stack{h} - (h_H^\ast + h_M^\ast - \Stack{h}) \Stack{m}\big)}{(m_H^\ast - \Stack{m})(m_M^\ast - \Stack{m})} \\
&\qquad +\frac{(h_H^\ast - \Stack{h})(h_M^\ast - \Stack{h})}{(m_H^\ast - \Stack{m})(m_M^\ast - \Stack{m})}\\
\end{align*}
and $A_H-B_H A_M^{-1} B_M$ must be positive definite.

The remaining parameter to be chosen is $D_M$. It must satisfy the following conditions:
\eqn{
(A_H A_M - D_H D_M)^2-4 (A_M B_H - B_M D_H) (A_H B_M -B_H D_M) \geq0,\\
(A_M B_H - B_M D_H)(A_H B_M - B_H D_M)\neq 0
}
The  CCVE is determined by the solution of two quadratic equations. The policy slopes for each agent are
\eqn{
\Ccve{L}_{H} &= \frac{
   D_H D_M - A_H A_M \pm \sqrt{
   4 (A_M B_H - B_M D_H) (B_H D_M-A_H B_M ) + (A_H A_M - D_H D_M)^2}}{
  2 A_H B_M - 2 B_H D_M}, \\
\Ccve{L}_{M} &=  \frac{D_H D_M - A_H A_M \pm \sqrt{
   4 (A_M B_H - B_M D_H) (B_H D_M-A_H B_M) + (A_H A_M - D_H D_M)^2}}{
  2 A_M B_H - 2 B_M D_H}, \\
}
  and the actions are 
\eqn{
\begin{bmatrix}
\Ccve{h} \\
\Ccve{m}
\end{bmatrix}
&=
\begin{bmatrix}
A_H + \Ccve{L}_{M}  B_H^\tops & B_M +\Ccve{L}_{M} D_H	 \\
B_M + \Ccve{L}_{H} D_H & A_M + \Ccve{L}_{H} B_M 
\end{bmatrix}^{-1}
\begin{bmatrix}
b_H+\Ccve{L}_{M} d_H \\
b_M+\Ccve{L}_{H} d_M
\end{bmatrix}
}
The reverse Stackelberg equilibrium is determined by policy slopes
\eqn{
\RStack{L}_{H} = \frac{h_H^\ast-h_M^\ast}{m_H^\ast - m_M^\ast},\ 
\RStack{L}_{M}  = -\frac{A_H \RStack{L}_{H} + B_H}{B_H \RStack{L}_{H} + D_H},\ 
}
and actions
$
\RStack{h} = \Opt{h}_M,\ 
\RStack{m} = \Opt{m}_M.
$
\newpage

\newpage
\section{Analysis of the quadratic game from the main paper}
\label{sec:theory-scalar}
This section provides mathematical statements about the two-player game $(c_H,c_M)$ with each player having an objective to optimize the functions:
\begin{equation*}
  \costH(\actH,\actM)=\tfrac{1}{2}\actH^2 + \tfrac{7}{30}\actM^2 -\tfrac{1}{3}hm + \tfrac{2}{15}h-\tfrac{22}{75}m  + \tfrac{12}{125}.
  \tag{\ref{eq:costH}}
\end{equation*}
\vspace{-1em}
for the human and
\begin{equation*}
\costM(\actH,\actM)= \tfrac{1}{2}m^2+h^2-hm.	
  \tag{\ref{eq:costM}}
\end{equation*}
for the machine.
In Experiment 1, the machine optimizes its action by gradient descent. In Experiment 2, the machine optimizes its policy by conjectural variations. In Experiment 3, the machine optimizes its policy by gradient descent. In all experiments, the human updates its action $h$ by making the cost $c_H(h,m)$ as small as possible.

In this section, the three main experiments from the paper were analyzed. 
Outcomes were predicted by the equilibrium solutions of coupled optimization problems. 
The three subsections contain mathematical propositions proving statements about the three respective experiments.
Propositions 
\ref{prop:exp1-alpha0} and
\ref{prop:exp1-alpha1} apply to Experiment 1. They prove convergence to the unique Nash and Stackelberg equilibrium solutions.
Propositions
\ref{prop:exp-affine},
\ref{prop:exp2-cv},
\ref{prop:exp2-brm},
\ref{prop:exp2-ccve} and
\ref{prop:exp2-limit}
apply to Experiment 2. They prove that the machine can perturb its own policy to estimate the human's conjectural variation, and in turn use the estimate to form a best response iteration that converges to a 
consistent conjectural variations equilibrium.
Propositions 
\ref{prop:exp3-linear},
\ref{prop:exp3-rse},
\ref{prop:exp3-grad},
\ref{prop:exp3-pg}
 apply to Experiment 3. They prove that the machine can perturb its own policy to estimate 
its policy gradient, and in turn use the estimate to update its policy to converge to its global optimum. The formal definitions of the equilibrium solutions are stated in Section \ref{sec:theory-def}.

A \emph{human-machine co-adaptation game} is a two-player repeated game determined by two cost functions -- one for each player.
 The game is played as follows:  at each time step $t$, the human chooses action $\actH_t\in\ActH$. The machine best responds by choosing action $\actM_t\in\ActM$. The human observes cost $c_H(h_t,m_t)$ via the interface.
The next action pair $(h_{t+1},m_{t+1})$ is chosen at the next time step $t+1$ for a fixed number of steps $T$. 
In each of our experiments, the method that the machine uses to update its action is varied.

\subsection{Experiment 1: gradient descent in action space}
\label{sec:scalar-exp1}

The following
Proposition \ref{prop:exp1-alpha0} describes the $\alpha=0$ case of Experiment 1,  where the outcome is the unique stable Nash  equilibrium of the game is $(m,h)=(-1/5,-1/5)$. This outcome is observed empirically (Figure 2 of main paper).

\begin{proposition}
\label{prop:exp1-alpha0}
Given a human-machine co-adaptation game determined by cost functions \eqref{eq:costH} and \eqref{eq:costM}, 
if the machine's action is $m=-1/5$, then the human's best response is $h=-1/5$.
\end{proposition}

\begin{proof}
From the human's perspective, the goal was to solve the optimization problem
\eqnn{scalar-optH-m}{
\min_h\ c_H(h,m)
}
The second order condition of \eqref{eq:scalar-optH-m} is
\eqn{
\partial_h^2c_H(h,m)&=1>0.
}
The first order condition of the optimization problem \eqref{eq:scalar-optH-m} is
\eqnn{scalar-foH-m}{
\partial_hc_H(h,m)&=h-\tfrac{1}{3}m+\tfrac{2}{15}=0.
}
By solving for $h$ in \eqref{eq:scalar-foH-m}, the human's best response to $m$ is 
\eqn{
h=\tfrac{1}{3}m-\tfrac{2}{15}.
}
 Solving for $h$ gives the human's best response $h=\tfrac{1}{3}m-\tfrac{2}{15}$. Thus, if $m=-\frac{1}{5}$, then $h=-\frac{1}{5}$.
\end{proof}
The following
Proposition \ref{prop:exp1-alpha1} describes the $\alpha=1$ (or ``infinity'')  case of Experiment 1, where the
 outcome is the unique stable human-led Stackelberg equilibrium of the game at $(m,h)=(1/5,1/5)$. This outcome is observed empirically (Figure 2 of main paper).

\begin{proposition}
\label{prop:exp1-alpha1}
Given a human-machine co-adaptation game determined by cost functions \eqref{eq:costH} and \eqref{eq:costM}, 
if the machine's policy is $m=h$, then the human's best response is $h=1/5$.
\end{proposition}
\begin{proof}
From the human's perspective, the optimization problem is
\eqnn{scalar-optH-h}{
\min_h \{c_H(h, m) \mid m=h\}
}
The cost experienced by the human is
\eqn{
c_H(h,h)=\tfrac{2}{5}h^2-\tfrac{4}{25}h+\tfrac{12}{125}
}
The first order condition of \eqref{eq:scalar-optH-h} is
\eqn{
\partial_hc_H(h,h)=\tfrac{4}{5}h-\tfrac{4}{25}=0
}
Solving for $h$  gives $h=\frac{1}{5}$.
\end{proof}

\begin{remark}
\label{prop:exp1-finite}
Given a human-machine co-adaptation game determined by cost functions \eqref{eq:costH} and \eqref{eq:costM}, 
if $0<\alpha\leq 1$ and the machine updates its action $m_{t+1}=m_t-\alpha \partial_m c_M(h_t,m_t)$, then $m_{t+1}$ approaches $h_t$ as $t$ increases.
This result can be shown by writing the update as $m_{t+1}=(1-\alpha)m_t+\alpha h_t$ showing that the sequence $m_t,m_{t+1},\dots$ is generated by an exponential smoothing filter of time-varying signal $h_t$. 
\end{remark}

Remark
\ref{prop:exp1-finite}
is observed in the 2D histograms in Figure 2 from the main paper as the distribution of points on the line of equality $m=h$ for larger $\alpha$ values.

\subsection{Experiment 2: conjectural variation in policy space}
\label{sec:scalar-exp2}

In Experiment 2, the machine iterated conjectural variations in policy space. 
From the humans's perspective, the goal was to choose $h$ to optimize $c_H(h,m)$. But how $m$ is determined affects the solution of the coupled optimization problems. 
From the machine's perspective, the goal was to choose $m$ to optimize $c_M(h,m)$. Similarly, what $h$ is assumed to be affects the machine's response. 
The machine estimates the conjectural variation that describes how $h$ is affected by a change in $m$.

The following
Proposition \ref{prop:exp-affine} describes 
the machine's policy perturbation in Experiment 1. The human's response is linear in the machine's constant perturbation $\delta$, but non-linear in the machine's policy slope $L$.

\begin{proposition}
\label{prop:exp-affine}
Given a human-machine co-adaptation game determined by cost functions \eqref{eq:costH} and \eqref{eq:costM}, 
if  the machine's policy is $m=Lh+\delta$ and $L$ satisfies $\tfrac{7}{15}L^2- \tfrac{2}{3}L+1 >0$, then the human's best response is 
\eqn{
	h=\frac{22L-10 -(35L-25)\delta}{ 35 L^2 - 50L + 75}
}
\end{proposition}
\begin{proof}
The human's optimization problem is
\eqnn{scalar-optH-L}{
\min_h\ \{c_H(h,m)\mid m=L h+\delta\}
}
The second order condition of \eqref{eq:scalar-optH-L} is
\eqn{
\tfrac{7}{15}L^2- \tfrac{2}{3}L+1 >0.
}
The first order condition of \eqref{eq:scalar-optH-L} is
\eqn{
(\tfrac{7}{15}L^2- \tfrac{2}{3}L+1 )h 
- \tfrac{22}{75}L+\tfrac{2}{15} 
- (\tfrac{7}{15}L_M+ \tfrac{1}{3} )\delta=0
}
Solving 
for $h$ gives the result.

\end{proof}

The following Proposition 
\ref{prop:exp2-cv} 
describes how the machine estimates the slope of the human's policy using two points generated by perturbing the constant term of the machine's policy.

\begin{proposition}
\label{prop:exp2-cv}
Given a human-machine co-adaptation game determined by cost functions \eqref{eq:costH} and \eqref{eq:costM}, 
if the machine's policies are $m=Lh$ and $m'=Lh'+\delta$ and the human best responds with $h$ and $h'$, then 
\eqn{
\frac{h'-h}{m'-m}=\frac{7L-5}{5L-15}
} 
\end{proposition}
\begin{proof}
Using Proposition \ref{prop:exp-affine} for $h'$ and $h$, 
\eqn{
h'-h=-\frac{35L-25}{35L^2-50L+75}\delta.
}
 Using the definitions of $m'$ and $m$, 
 \eqn{
m'-m=L(h'-h)+\delta.
}
The ratio of the differences is therefore 
\eqn{
\frac{h'-h}{m'-m}=\frac{-\paren{\frac{35L-25}{35L^2-50L+75}\delta}}{-L\paren{\frac{35L-25}{35L^2-50L+75}\delta} + \delta}
=
\frac{{35L-25}}{L(35L-25)-(35L^2-50L+75)}=\frac{7L-5}{5L-15}.
}
\end{proof}

\begin{remark}
\label{remark:LHLM}
In the main paper, the human's policy slope is $L_H$ and the machine's policy slope is $L_M$. For a machine policy $m=Lh$ in Experiments 2 and 3, the relationship between these terms are 
\eqn{
L_M&=L,\\
L_H&=\frac{7L-5}{5L-15}.
}	
In this case, the human's conjecture of the machine is consistent with the machine's policy.
The equilibrium solutions are described by linear equations
\eqn{
m&=L_Mh+\ell_M\\
h&=L_Hm+\ell_H
}
where $\ell_M=0$ and $\ell_H=-\frac{22L-10}{25L-75}.$
\end{remark}
Remark \ref{remark:LHLM} can produce the curves seen in \figref{pareto-consistency} as the solid-line ellipse for when $H$ has a consistent conjecture about $M$ by sweeping $L$ along the real line.

The following Proposition 
\ref{prop:exp2-brm} 
describes  the machine's best response to the human adopting a policy based on the conjectural variation in  Proposition \ref{prop:exp2-cv}.
\begin{proposition}
\label{prop:exp2-brm}
Given a human-machine co-adaptation game determined by cost functions \eqref{eq:costH} and \eqref{eq:costM}, 
if the human's policy is $h=\paren{\frac{7L-5}{5L-15}}m+\ell$ for some $\ell$, then the machine's best response is 
\eqn{
m=\frac{9L+5}{2L+10}h
} 
\end{proposition}
\begin{proof}
The machine's optimization problem is
\eqnn{scalar-optM-Lest}{
\min_m \set{ c_M(h,m)\mid h=\paren{\tfrac{7L-5}{5L-15}}m+\ell }.
}
The first order condition of \eqref{eq:scalar-optM-Lest} is
\eqnn{scalar-foM-Lest}{
\partial_m c_M(h,m)+\partial_h c_M(h,m)\paren{\tfrac{7L-5}{5L-15}}=0.
}
The second order condition is
\eqn{
2\paren{\tfrac{7L-5}{5L-15}}^2 - 2\paren{\tfrac{7L-5}{5L-15}} + 1>0.
}
Taking the first order condition in
\eqref{eq:scalar-foM-Lest}, 
the equation is
\eqn{
m-h+(2h-m)\paren{\tfrac{7L-5}{5L-15}}=0
}
Sovling for $m$ gives the machine's best response
\eqn{
m=\frac{9L+5}{2L+10}h
}
\end{proof}

\begin{remark}
The constant term $\ell$ in Proposition \ref{prop:exp2-brm} can be estimated from the joint action measurements. However, it is not necessary to do so to arrive at the optimality condition in Equation~\eqref{eq:scalar-foM-Lest}.
	
\end{remark}

The following Proposition~\ref{prop:exp2-ccve} shows the existence of a consistent conjectural variations equilibrium. The equilibrium solution concept is defined in \secref{theory-def}. It describes the situatuion where both players have consistency of actions and policies.

\begin{proposition}
\label{prop:exp2-ccve}
Given a human-machine co-adaptation game determined by cost functions \eqref{eq:costH} and \eqref{eq:costM}, 
there exists two consistent conjectural variations equilibrium solutions uniquely defined by the machine response slopes
\eqn{
L=\frac{-1\pm\sqrt{41}}{4}.
}
\end{proposition}
\begin{proof}
Using Equations (1) and (1') from Definition 4.10 in \cite{bacsar1998dynamic}, the stationary conditions for a consistent conjectural variation in the policy space is
\eqnn{exp3-riccatti-subbed}{
L-L\paren{\tfrac{7L-5}{5L-15}}+2\paren{\tfrac{7L-5}{5L-15}}-1=0,\\
}
Simplifying the numerator of \eqref{eq:exp3-riccatti-subbed}, the following quadratic equations defines the machine's consistent policy slope: 
\eqn{
2 L^2+ L-5  =0.
}
The solution to the quadratic equation gives us the result. 
\end{proof}
\begin{remark}
The human's policy slope can be determined by substituting in 
$
L = \frac{-1\pm\sqrt{41}}{4}
$, which results in
\eqn{
\frac{7L-5}{5L-15}=\frac{1\mp\sqrt{41}}{10}.
}
So the two consistent conjectural variational policies are
\eqn{
m&= \frac{-1\pm\sqrt{41}}{4} h \\
h&=\frac{1\mp\sqrt{41}}{10} m- \frac{3 + 7 \sqrt{41}}{100} 
}
and the actions $(m,h)$ that solve the linear equation.
\end{remark}

The following Proposition~\ref{prop:exp2-limit} shows that Experiment 2 converges to a stable equilibrium.
\begin{proposition}
\label{prop:exp2-limit}
Given a human-machine co-adaptation game determined by cost functions \eqref{eq:costH} and \eqref{eq:costM}, 
if the machine updates its policy using the difference equation $L^+=\frac{9L+5}{2L+10}$ then 
\eqn{
L^* =\frac{-1+\sqrt{41}}{4}
} 
is a locally exponentially stable fixed point of this iteration.
\end{proposition}
\begin{proof}
Define the map $F:\R\to\R$ as
\eqnn{scalar-lft}{
F(L)
:=\frac{9L+5}{2L+10}
}
To assess the convergence of Experiment 2, 
the fixed points of \eqref{eq:scalar-lft} are determined along with their stability properties. The fixed point $L^*$ that satisfies 
\eqn{
L^*=F(L^*)
} are determined by the solutions to the quadratic equation
\eqnn{scalar-characteristic-eq}{
2L^2+L-5=0.
}
There are two solutions to \eqref{eq:scalar-characteristic-eq} and they are real and distinct.
The fixed points are 
\eqn{
\frac{-1\pm\sqrt{41}}{4}.
}
Exactly one fixed point is stable, and it is a stable attractor of the repeated application of $F$.
The stability can be determined by linearizing
\eqref{eq:scalar-lft} at the particular fixed point and ensuring that its derivative gives a magnitude of less than one. The linearization of $F$ at fixed point $L^*$ is
\eqnn{scalar-cv-linearized}{
F(L)\approx \partial F(L^*)(L-L^*)
}
where 
\eqn{
\partial F(L)=\frac{20}{(5+L)^2}
}
If $L^*=\frac{-1+\sqrt{41}}{4}$, then $|\partial F(L^*)|\approx0.5<1$, so the fixed point $L^*$ is  stable. On the other hand,
if $L^*=\frac{-1-\sqrt{41}}{4}$, then $|\partial F(L^*)|>1$, so the fixed point $L^*$ is  unstable.
\end{proof}
For a quadratic game with single-dimensional actions, there are two consistent conjectural variations equilibria. One is stable, the other is unstable.

\begin{remark}
\label{remark:complex-analysis}
Another way to assess the convergence of the fixed point map \eqref{eq:scalar-lft} is by inspecting the normal form of the linear fractional transformation. The normal form of \eqref{eq:scalar-lft} is
\eqnn{scalar-lft-normal}{
\frac{F(L)-L^*}{F(L)-L^{**}}
&=\lambda \frac{L-L^*}{L-L^{**}}\\
}
where $L^*$ and $L^{**}$ are fixed points of $F$ and $\lambda$ is a real number given by
\eqnn{scalar-lft-coeff}{
\lambda=\frac{-19+\sqrt{41}}{-19-\sqrt{41}}\\
}
Since $|\lambda|\approx 0.5<1$, the fixed point  $L^*$ is semi-globally stable.	
\end{remark}
Remark~\ref{remark:complex-analysis} is a based on a known result from complex analysis and conformal mapping theory.

\subsection{Experiment 3: gradient descent in policy space}
\label{sec:scalar-exp3}

In Experiment 3, the machine implemented gradient descent in policy space. 
The machine estimated the policy gradient using cost measurements from a pair of trials. The machine's cost depends on its own policy and the human's best response to it.

The following Proposition~\ref{prop:exp3-linear} describes the machine's policy perturbation in Experiment 3. The human's action response varies non-linearly. 

\begin{proposition}
\label{prop:exp3-linear}
Given a human-machine co-adaptation game determined by cost functions \eqref{eq:costH} and \eqref{eq:costM}, 
if the machine's policy is $m=(L+\Delta)h$ 
and $L,\Delta$ satisfy $\tfrac{7}{15}(L+\Delta)^2- \tfrac{2}{3}(L+\Delta)+1 >0$,
then the human's best response is 
	\eqn{
	h=\frac{22(L+\Delta)-10}{ 35 (L+\Delta)^2 - 50(L+\Delta) + 75}
 }
\end{proposition}
\begin{proof}
The human's optimization problem is
\eqnn{scalar-optH-D}{
\min_{h}\ \{ c_H(h,m)\mid m=(L+\Delta)h\}.
}
The second order condition of \eqref{eq:scalar-optH-D} is
\eqn{
\tfrac{7}{15}(L+\Delta)^2- \tfrac{2}{3}(L+\Delta)+1 >0.
}
The first order condition of \eqref{eq:scalar-optH-D} is
\eqn{
(\tfrac{7}{15}(L+\Delta)^2- \tfrac{2}{3}(L+\Delta)+1 )h 
- \tfrac{22}{75}(L+\Delta)+\tfrac{2}{15}  =0
}
Solving for $h$ gives human's response 
\eqnn{scalar-brH-D}{
h=\frac{22(L+\Delta)-10}{ 35 (L+\Delta)^2 - 50(L+\Delta) + 75}.
}
\end{proof}

The following Proposition~\ref{prop:exp3-grad} describes how to estimate the policy gradient using two trials as done in Experiment 3. 
Suppose the machine plays policy $m=Lh$, then the human's response is given by
\eqn{
r(L):=\frac{22L-10}{ 35 L^2 - 50L + 75}
}
as determined by Proposition~\ref{prop:exp-affine} or Proposition~\ref{prop:exp3-linear} with the perturbations set to zero.
\begin{proposition}
\label{prop:exp3-grad}
Given a human-machine co-adaptation game determined by cost functions \eqref{eq:costH} and \eqref{eq:costM}, 
if the machine's policies are $m=Lh$ and $m'=(L+\Delta)h'$ and the human's best responses are $h=r(L)$ and $h'=r(L+\Delta)$,
then 
\eqn{
\lim_{\Delta\to 0}\frac{c_M(h',m')-c_M(h,m)}{\Delta}= D_Lc_M(r(L),Lr(L))
}

\end{proposition}
\begin{proof}
From Proposition~\ref{prop:exp-affine}, if machine's policy is $m=Lh$ and the human's best response is
\eqn{
h=\frac{22L-10}{ 35 L^2 - 50L + 75}.
} 
The machine's cost written as a function of $L$ is
\eqn{
c_M(h,m)=c_M(r(L),Lr(L))
&= \tfrac{1}{2} L^2r(L)^2 + r(L)^2 - L r(L)^2 \\
&= \tfrac{1}{2}(L^2-2L+2)r(L)^2\\
&= \frac{(L^2-2L+2)(22L-10)^2}{2( 35 L^2 - 50L + 75)^2}
}
The difference term is
\eqn{
c_M(h',m')-c_M(h,m)=c_M(r(L+\Delta),Lr(L+\Delta))-c_M(r(L),Lr(L))
}
Expanding out the terms, ignoring the terms of order $\Delta^2$ or higher, we have
\eqn{
c_M(h',m')-c_M(h,m)&=\frac{((L+\Delta)^2-2(L+\Delta)+2)(22(L+\Delta)-10)^2}{2( 35 (L+\Delta)^2 - 50(L+\Delta) + 75)^2}
-\frac{(L^2-2L+2)(22L-10)^2}{2( 35 L^2 - 50L + 75)^2}\\
& 
=\frac{4( 11 L-5) (2 L^3 + 181 L^2 - 380 L + 305  )}{25 (7 L^2 - 10 L + 15 )^3}\Delta  + \paren{\cdots}\Delta^2 + \cdots}
Dividing by $\Delta$ and taking $\Delta $ to zero gives us the same expression as
directly computing the derivative of the  cost:
\eqn{
\partial_L c_M(r(L),Lr(L)) = \frac{4( 11 L-5) (2 L^3 + 181 L^2 - 380 L + 305  )}{25 (7 L^2 - 10 L + 15 )^3}.
}
Hence, we get the desired result.
\end{proof}

The following Proposition~\ref{prop:exp3-rse} shows that there is a unique machine-led reverse Stackelberg equilibrium of the game. The equilibrium solution concept is defined in~\secref{theory-def}. It describes the scenario where the leader announces a policy and the follower responds to the policy. In contrast, the Stackelberg equilibrium in Proposition~\ref{prop:exp1-alpha1} describes the scenario where the leader announces an action and the follower response to the action.

\begin{proposition}
\label{prop:exp3-rse}
Given a human-machine co-adaptation game determined by cost functions \eqref{eq:costH} and \eqref{eq:costM}, 
there exists a reverse Stackelberg equilibrium. 
\end{proposition}
\begin{proof}
The machine's global optimum solves
\eqn{
\min_{h,m} c_M(h,m).
}
The machine's global optimum is $(h,m)=(0,0)$.

Suppose the machine's policy is $m=Lh$, then the human's optimization problem is
\eqn{
\min_h \{ c_H(h,m)\mid m=Lh \}
}
and the best response is
\eqn{
	h=r(L)=\frac{22L-10}{ 35 L^2 - 50L + 75}
}
The machine wants to drive the human to play $0=r(L)$. Hence
the machine chooses $L=5/11$. 

The second order condition is
\eqn{
\tfrac{7}{15}L^2- \tfrac{2}{3}L+1 >0.
}
which is satisfied by $L=5/11$. Hence $(0,0)$ is a machine-led reverse Stackelberg equilibrium.
\end{proof}

The following Proposition~\ref{prop:exp3-pg} shows that Experiment 3 converges to a stable equilibrium.

\begin{proposition}
\label{prop:exp3-pg}
Given a human-machine co-adaptation game determined by cost functions \eqref{eq:costH} and \eqref{eq:costM}, 
if the machine plays policy $m=Lh$ and the human responds with 
$h=r(L)$ and
machine's updates its policy by gradient descent,
\eqn{
L_{k+1}=L_k-\alpha \partial_L c_M(r(L_k),L_kr(L_k))
}
then $L^* = 5/11$ is a locally exponentially stable fixed point of this iteration for all $\alpha > 0$ sufficiently small.

\end{proposition}
\begin{proof}
 The roots 
 of $\partial_L c_M(r(L_k),L_kr(L_k))=0$
 are determined by the solutions to a quartic equation
\eqnn{scalar-pg-eq}{
( 11 L-5) (2 L^3 + 181 L^2 - 380 L + 305  )=0.
}
There are two real solutions to \eqref{eq:scalar-pg-eq}, the first one $L^*=\tfrac{5}{11}$  can be seen by inspection, 
and the second one is, approximately, 
$L^{**}\approx-92.6.$

The stability is determined by linearizing at the particular fixed point and ensuring that the second derivative is positive. The linearization the derivative at root $L_M^*$ is
\eqnn{scalar-pg-linearized}{
\partial_{L}c_M(r(L),Lr(L))\approx \partial_{L}^{\ 2} c_M(r(L^*),L^*r(L^*))(L-L^*)}
The second derivative $\partial_{L_M}^{\ 2}c_M\approx 0.18$ evaluated at $L^*$ is positive, so the fixed point $L_M^*$ is stable. 
The second derivative evaluated at $L^{**}$ is negative, so the fixed point is unstable.
\end{proof}

\section{Interpretations of consistent conjectural variations}
\label{sec:interpretation}
In this section, interpretations of the consistency conditions with regards to conjectural variations are provided. They relate to partial differential equations that arise in economics and non-cooperative dynamic games. 
\subsection{Comparative statics}
\label{sec:interpretation-comparative-statics}

A quintessential microeconomics tool, \emph{comparative statics} (or \emph{sensitivity analysis} more generally) is a technique for comparing economic outcomes given a change in an exogenous parameter or \emph{intervention} \cite{varian1992microeconomic}.
If the expression $f(x,y)=0$ defines the equilibrium conditions for an economy where $x$ is an endogenous parameter (e.g., price of a product) and $y$ is an exogenous parameter (e.g., demand for a product), then up to first order the change in $x$ caused by a (small) change in $y$ must satisfy $\partial_xf\cdot \mathrm{d}x+\partial_y f\cdot \mathrm{d}y=0$, and under sufficient regularity, we may write $\mathrm{d}x/\mathrm{d}y=-(\partial_xf)^{-1}\cdot\partial_y f$. Comparative statics can also be applied to equilibrium conditions for an optimization problem.

This is precisely how it is used here: comparative statics analysis is applied to the first-order optimality conditions for 
\eqnn{statics-opt-h}{
\arg\min_m\{c_H(h,m)\mid m=\pi_M(h)\}
}
wherein the machine's action is treated as the intervention. Specifically, given an affine policy $\polM(\actH)=\ConjM\actH+\conjM$ and \eqref{eq:statics-opt-h}, we use this microeconomics analysis tool to understand how changes in $\actM$ induce changes in $\actH$ that are consistent with the optimality conditions of \eqref{eq:statics-opt-h}. 
This leads to a  process by which we derive an expression for the human's (best-)response in terms of the policy parameters $(\ConjM,\conjM)$ and the machine's corresponding action $\actM$.
First-order optimality conditions for \eqref{eq:statics-opt-h} are given by
\eqnns{firstorder}{
    0&=\partial_\actH\costH(\actH,\polM(\actH))|_{\polM(h)=\actM} + \partial_\actM\costH(\actH,\polM(\actH))|_{\polM(\actH)=\actM}\cdot \partial_\actH \polM(\actH),\\
    &=\partial_\actH\costH(\actH,\polM(\actH))|_{\polM(h)=\actM} + \partial_\actM\costH(\actH,\polM(\actH))|_{\polM(\actH)=\actM}\cdot \ConjM.
}
Using comparative statics as described above, 
we have that
\eqnn{cshuman}{
0=\partial_\actH^2\costH(\actH,\actM) \mathrm{d}\actH +\partial_{\actH\actM}^2\costH(\actH,\actM)\mathrm{d}\actM+( \partial_{\actH\actM}^2\costH(\actH,\actM)\mathrm{d}\actH+\partial_\actM^2\costH(\actH,\actM)\mathrm{d}\actM)\ConjM.
}
Hence, we deduce that
\eqnns{humanconj}{
\ConjH:=\frac{\mathrm{d}\actH}{\mathrm{d}\actM}&=-(\partial^2_\actH \costH+\partial_{\actH\actM}\costH\cdot\ConjM)^{-1}(\partial_{\actH\actM}\costH+\ConjM^\top \cdot \partial^2_\actM\costH),\\
&=-(A_H+\ConjM^\top B_H )^{-1}(B_H+\ConjM^\top D_H).
}
In Experiment 2, we will see a procedure for estimating the human's response $\widehat{\actH}$ as a function of $\actM$ by affinely perturbing $\polM(\actH)=\ConjM\actH+\conjM$. 
The machine then uses the estimate for the human's response as its conjecture in
\eqnn{machineresponse}{
\arg\min_{\actM}\{\costM(\actH,\actM)|\ \actH={\ConjH}\actM+{\conjH}\}
}
and obtain the policy it should implement at the next level.\newline

\subsection{Order of consistency via Taylor series approximation}
\label{sec:interpretation-basar}
\renewcommand{\Ccve}[1]{{#1}^{\tt c}}

Basar and Olsder \cite{bacsar1998dynamic} derives different orders of consistent conjectural variations equilibrium by taking the Taylor expansion of a conjecture to the cubic order. Let $(\Ccve{h},\Ccve{m})$ be the consistent conjectural variations equilibrium, $(\Ccve{L}_H,\Ccve{L}_M)$ be the consistent conjecture policy slopes. Let $\Ccve{\ell}_H=\Ccve{h}-\Ccve{L}_H\Ccve{m}$ and $\Ccve{\ell}_M=\Ccve{m}-\Ccve{L}_M\Ccve{h}$. The first order representation of a conjecture, that is an affine conjecture
\eqn{
\Ccve{h}&\approx \Ccve{L}_H m + \Ccve{\ell}_H+ \mc O(m^2),\\
\Ccve{m}&\approx \Ccve{L}_M h + \Ccve{\ell}_M + \mc O(h^2)} 
The 
partial differential equations that describe stationarity are
\eqn{
\frac{\partial c_H(h,m)}{\partial h} + \frac{\partial c_H(h,m)}{\partial m} \cdot \frac{\partial (\Ccve{L}_Mh+\Ccve{\ell}_M)}{\partial h}&=0,\text{ for }h=\Ccve{L}_Hm+\Ccve{\ell}_H,\\
\frac{\partial c_M(h,m)}{\partial m} + \frac{\partial c_M(h,m)}{\partial h} \cdot \frac{\partial (\Ccve{L}_Hm+\Ccve{\ell}_H)}{\partial m}&=0,\text{ for }m=\Ccve{L}_Mh+\Ccve{\ell}_M,\\
}
Writing what basar calls the ``first-order'' CCVE has stationarity conditions
\eqn{
\textstyle
\frac{\partial^2 c_H}{\partial h^2}
\cdot
\frac{\partial (\Ccve{L}_Hm+\Ccve{\ell}_H)}{\partial m}
+
\frac{\partial^2 c_H}{\partial h \partial m}
\paren{1+\frac{\partial (\Ccve{L}_Hm+\Ccve{\ell}_H)}{\partial m}\cdot \frac{\partial (\Ccve{L}_Mh+\Ccve{\ell}_M)}{\partial h}}
+\frac{\partial^2 c_H}{\partial m^2}
\cdot
\frac{\partial (\Ccve{L}_Mh+\Ccve{\ell}_M)}{\partial h}
=0,\\
\textstyle
\frac{\partial^2 c_M}{\partial m^2}
\cdot
\frac{\partial (\Ccve{L}_Mh+\Ccve{\ell}_M)}{\partial h}
+
\frac{\partial^2 c_M}{\partial m \partial h}
\paren{1+ \frac{\partial (\Ccve{L}_Mh+\Ccve{\ell}_M)}{\partial h}\cdot \frac{\partial (\Ccve{L}_Hm+\Ccve{\ell}_H)}{\partial m}}
+\frac{\partial^2 c_M}{\partial h^2}
\cdot
\frac{\partial (\Ccve{L}_Hm+\Ccve{\ell}_H)}{\partial m}
=0,
}
with arguments at $(h,m)=(\Ccve{h},\Ccve{m})$.
Hence
\eqn{
A_H\Ccve{L}_H+B_H(1+\Ccve{L}_H\Ccve{L}_M)+D_H\Ccve{L}_M=0,\\
A_M\Ccve{L}_M+B_M(1+\Ccve{L}_M\Ccve{L}_H)+D_M\Ccve{L}_H=0,\\
}
Solving for $\Ccve{L}_H,\Ccve{L}_M$ from the above equations gives
\eqn{
\Ccve{L}_H=-\frac{B_H+\Ccve{L}_M D_H}{A_H+\Ccve{L}_MB_M},\\
\Ccve{L}_M=-\frac{B_M+\Ccve{L}_H D_M}{A_M+\Ccve{L}_HB_H}
}
which shows that $\Ccve{L}_H,\Ccve{L}_M$ are fixed points of the conjectural iteration.


\newpage 
\section*{Extended data sections}
The 
additional methods are in
\secref{supp-mat}. The details on Experiments 1, 2 and 3 are in
\secref{supp-exp1},
\secref{supp-exp2},
and \secref{supp-exp3}.
Numerical simulations of the adaptive algorithms used in Experiments 1, 2 and 3
are in \secref{supp-sim}.
The experiments are shown to be generalizable through additional experiments  in
\secref{additional-experiments}, where experiment parameters and cost structures are varied.
The user study task load survey and feedback forms are provided in 
\secref{feedback}.
\appendix 
\renewcommand{\thesubsection}{\Alph{section}.\arabic{subsection}}
\setcounter{section}{0}

\section{Additional Methods}

\label{sec:supp-mat}
Additional experiments, whose results are reported in this Supplement but not the main paper, were conducted with different quadratic and non-quadratic costs to demonstrate the generality of the experiment and theory.
{First (\secref{exp3-other-init}), Experiment 3 was repeated with a different initialization of the machine's policy: instead of initializing the machine's policy to $m=h$, it was initialized to $m=0$. 
Next (\secref{exp3-other-optima}), Experiment 3 was repeated 9 times with different global {\optima} for the machine:
the machine's quadratic cost re-parameterized as 
\eqn{
c_M(h,m)=\frac{1}{2}(m-m_M^\ast)^2 - (m-m_M^\ast)(h-h_M^\ast) 
+(h-h_H^\ast)^2
}
with
$h_M^\ast \in \{-0.1,0,+0.1\}$ and $m_M^\ast \in \{-0.1,0,+0.1\}$ to test
whether the machine can drive the behavior to any one of a finite set of 
points in the joint action space, and to test whether the reverse-Stackelberg equilibrium $(\RStack{h},\RStack{m})=(h_M^\ast, m_M^\ast)$ is a stable equilibrium of policy gradient.

Subsequently (\secref{exp-cobb}), Experiments 1, 2, and 3 were repeated with non-quadratic cost functions in the \emph{Cobb-Douglas} form (modified from the example in Section~C.2 of~\cite{Figuieres2004-ls}):
%
\eqnn{Cobb-Douglas:H}{
\costH(\actH,\actM) = 1-2(1-\actH)^{0.175}(\actH + 1.1 \actM)^{0.5}\ 
}
was used in replicates of Experiments 1, 2, and 3;
\eqnn{Cobb-Douglas:M}{
\costM(\actH,\actM) = 1-2(1 - \actM)^{0.2}(\actM + 1.1 \actH)^{0.5}
}
was used in replicates of Experiments 1 and 2, and
\eqnn{Cobb-Douglas:M:quad}{
c_M(h,m)=(m-m_M^\ast)^2+(h-h_M^\ast)^2\ \text{with}\ (m_M^\ast,h_M^\ast)=(0.5, 0.5)
}
was used in replicates of Experiment 3.
Pairing $\costH$ from~\eqref{eq:Cobb-Douglas:H} with $\costM$ from~\eqref{eq:Cobb-Douglas:M} yields the following game-theoretic equilibria in the replicates of Experiments 1 and 2:  
\eqn{
(\Nash{\actH},\Nash{\actM}) & \approx (0.590, 0.529),\\
(\Stack{\actH},\Stack{\actM}) & \approx (0.429, 0.579),\\
(\Ccve{\actH},\Ccve{\actM}) & \approx (0.392,0.336).
}
Pairing $\costH$ from~\eqref{eq:Cobb-Douglas:H} with $\costM$ from~\eqref{eq:Cobb-Douglas:M:quad} yields the following  equilibrium in the replicates of Experiment 3:  
\eqn{
(\RStack{\actH},\RStack{\actM})&=(0.5,0.5).
}
The human's actions were constrained to $[0.2,0.8]$ in these replicates of the experiments and the manual input was accordingly normalized to this range. 
The machine's actions were constrained to $[0,1]$. 
Experiment-specific changes to protocol designs are described in subsequent subsections.
}

\subsection{Experiment 1: gradient descent in action space} 
\label{sec:supp-exp1}

Protocol~S\ref{code:exp1} summarizes the procedure for Experiment 1.

{The preceding methods were modified as follows for the experiments with non-quadratic costs in~\secref{exp-cobb}: 
the policy implemented for the case $\alpha = \infty$ was $m=-\frac{77}{270}h+\frac{20}{27}$;
the joint action was initialized uniformly at random in the square $[0.3,0.7]\times[0.3,0.7]\subset\R^2$.
}

\subsection{Experiment 2: conjectural variation in policy space}  
\label{sec:supp-exp2}
Protocol~S\ref{code:exp2} summarizes the procedure for Experiment 2.

The preceding methods were modified as follows for the experiments with non-quadratic costs in \secref{exp-cobb}: 
given non-quadratic cost in Cobb-Douglas form 
\eqnn{cobbdouglas-cost}{c_M(h,m)=1-2(1-m)^{a_M} (m+d_Mh)^{b_M}} 
where $a_M,b_M>0$ and $d_M\geq 1$,
the machine's conjectural variation iteration is 
\eqnns{cobbdouglas-cv}{
L_{M,k+1} & =-\frac{a_Md_M}{a_M+b_M+b_M d_M L_H},\\
\ell_{M,k+1} & = \frac{b_M+b_Md_ML_H}{a_M+b_M+b_M d_M L_H}.
}

\subsection{Experiment 3: gradient descent in policy space} 
\label{sec:supp-exp3}
Protocol~S\ref{code:exp3} summarizes the procedure for Experiment 3. 

See Propositions  \ref{prop:exp3-grad} and \ref{prop:exp3-pg} in \secref{scalar-exp3} for the theoretical results on the policy gradient estimate and convergence.

\section{Additional experimental results}
\label{sec:additional-experiments}

Additional experiments were conducted with different quadratic and non-quadratic costs to demonstrate the generality of the experimental and theoretical results.

\subsection{Machine initialization (Experiment 3)}\label{sec:exp3-other-init}
To demonstrate that the outcome of the machine's policy gradient adaptation algorithm does not depend on the initialization of the machine's policy, we repeated Experiment 3 with initial policy slope to $L_M=0$. Iterating policy gradient shifted the distribution of median action vectors for a population of human subjects to the machine's global optimum (\figref{exp3-other-init}).

\subsection{Machine {\optimum} (Experiment 3)}\label{sec:exp3-other-optima}
To demonstrate that the machine can drive the human action to any point in the action space so long as the joint action profile is stable, the three experiments were conducted with differing machine minima. A grid of machine minima were tested $h_M^\ast\in\{-0.1,0,+0.1\}$ and $m_M^\ast\in\{-0.1,0,+0.1\}$. Iterating policy gradient descent shifted the distribution of median action vectors for a population of human subjects to the machine's global optimum (\figref{exp3-other-optima}).

\subsection{Non-quadratic costs (Modified Experiments 1, 2, and 3)}\label{sec:exp-cobb}
To demonstrate the generality of the experiments and theory, we conducted modified Experiments 1, 2 and 3 using non-quadratic costs. In Experiment 1, the distributions of median action vectors for a population of human subjects shifted from the Nash equilibrium at the slowest rate to the human-led Stackelberg equilibrium at the fastest adaptation rate (\subfigref{exp-cobb}{A}). In Experiment 2, iterating the process of estimating conjectural variations shifted the distribution of median action vectors for a population of human subjects from the human-led Stackelberg equilibrium to a consistent conjectural variations equilibrium (\subfigref{exp-cobb}{B}). In Experiment 3, iterating policy gradient descent shifted the distribution of median action vectors for a population of human subjects to the machine's global minimum (\subfigref{exp-cobb}{C}).

\subsection{Numerical simulations}
The three experiments were numerically simulated. The results from the simulation are overlaid on top of the violin data plots from the main paper (\figref{sim-quad}). In Experiment 1, the simulation captures the transition from the Nash equilibrium at the slowest rate to the human-led Stackelberg equilibrium at the fastest rate (\subfigref{sim-quad}{A}). In Experiment 2, the simulation captures the transition from the human-led Stackelberg equilibrium to the consistent conjectural variations equilibrium (\subfigref{sim-quad}{B}). In Experiment 3, the simulation captures the transition from the human-led Stackelberg equilibrium to the machine's global optimum (\subfigref{sim-quad}{C}).

\subsection{Consistency vs.\ Pareto-optimality}
To demonstrate that the equilibrium points reached in the experiments are not Pareto-optimal, 
except for the machine's global minimum, the sets are compared with the consistent conjecture conditions  (\figref{pareto-consistency}).
The Pareto-optimal set of actions solve 
\eqnn{Pareto}{\min_{h,m}\ \gamma c_H(h,m) + (1-\gamma) c_M(h,m)} 
for $\gamma$ between 0 and 1. 
See \cite{Debreu1954-ip} for the definition of Pareto optimality.
The consistency conditions are satisfied when one player's conjecture is equal to the other player's policy (see Definition 4.9 of \cite{bacsar1998dynamic}).
 The data from Experiments 2 and 3 from the main paper, and Experiment 3 with different initialization from~\secref{exp3-other-init} are plotted in \figref{pareto-consistency}. The data overlap the curve where the human's conjecture is consistent with the machine's policy.

\newpage

\begin{table}[t]
\footnotesize\centering
Results from statistical tests for Experiments 1, 2 and 3 with $P$-values, $t$-statistics, and Cohen's $d$.

\begin{tabular}{|lllll|}
\hline
Experiment 1 &&&& \\
\bf $H_0$: mean of initial Human action distribution is equal to $h^{\NE}$        & 
$P=0.20$	& $t=+1.3$	& $d=+0.2$ & \\
\bf $H_0$: mean of initial Machine action distribution is equal to $m^{\NE}$      & 
$P=1.00$	& $t=+0.0$	& $d=-1.0$ & \\
$H_0$: mean of initial Human action distribution is equal to $h^{\SE}$        & 
$P=0.00$	& $t=-26.9$	& $d=-4.2$ &  $\star$\\
$H_0$: mean of initial Machine action distribution is equal to $m^{\SE}$      & 
$P=0.00$	& $t=-\infty$	& $d=-\infty$ &  $\star$\\
$H_0$: mean of final Human action distribution is equal to $h^{\NE}$          & 
$P=0.00$	& $t=+21.2$	& $d=+3.4$ &  $\star$\\
$H_0$: mean of final Machine action distribution is equal to $m^{\NE}$        & 
$P=0.00$	& $t=+21.2$	& $d=+3.4$ &  $\star$\\
\bf $H_0$: mean of final Human action distribution is equal to $h^{\SE}$          & 
$P=0.49$	& $t=-0.7$	& $d=-0.1$ & \\
\bf $H_0$: mean of final Machine action distribution is equal to $m^{\SE}$        & 
$P=0.49$	& $t=-0.7$	& $d=-0.1$ & \\
\hline
Experiment 2 &&&& \\
\bf $H_0$: mean of initial Human action distribution is equal to $h^{\SE}$        & 
$P=0.24$	& $t=-1.2$	& $d=-0.3$ & \\
\bf $H_0$: mean of initial Machine action distribution is equal to $m^{\SE}$      & 
$P=0.24$	& $t=-1.2$	& $d=-0.3$ & \\
\bf $H_0$: mean of initial Human policy distribution is equal to $L_H^{\SE}$      & 
$P=0.10$	& $t=+1.7$	& $d=+0.4$ & \\
\bf $H_0$: mean of initial Machine policy distribution is equal to $L_M^{\SE}$    & 
$P=1.00$	& $t=+0.0$	& $d=\text{NaN}$ & \\
$H_0$: mean of initial Human action distribution is equal to $h^{\CCVE}$      & 
$P=0.00$	& $t=-10.0$	& $d=-2.3$ &  $\star$\\
$H_0$: mean of initial Machine action distribution is equal to $m^{\CCVE}$    & 
$P=0.00$	& $t=-21.3$	& $d=-4.9$ &  $\star$\\
$H_0$: mean of initial Human policy distribution is equal to $L_H^{\CCVE}$    & 
$P=0.00$	& $t=+12.1$	& $d=+2.8$ &  $\star$\\
$H_0$: mean of initial Machine policy distribution is equal to $L_M^{\CCVE}$  & 
$P=0.00$	& $t=-\infty$	& $d=\text{NaN}$ &  $\star$\\
$H_0$: mean of final Human action distribution is equal to $h^{\SE}$          & 
$P=0.00$	& $t=+4.9$	& $d=+1.1$ &  $\star$\\
$H_0$: mean of final Machine action distribution is equal to $m^{\SE}$        & 
$P=0.00$	& $t=+7.6$	& $d=+1.7$ &  $\star$\\
$H_0$: mean of final Human policy distribution is equal to $L_H^{\SE}$        & 
$P=0.00$	& $t=-6.4$	& $d=-1.5$ &  $\star$\\
$H_0$: mean of final Machine policy distribution is equal to $L_M^{\SE}$      & 
$P=0.00$	& $t=+13.0$	& $d=+3.0$ &  $\star$\\
\bf $H_0$: mean of final Human action distribution is equal to $h^{\CCVE}$        & 
$P=0.02$	& $t=-2.6$	& $d=-0.6$ &  $\star$\\
\bf $H_0$: mean of final Machine action distribution is equal to $m^{\CCVE}$      & 
$P=0.02$	& $t=-2.5$	& $d=-0.6$ &  $\star$\\
\bf $H_0$: mean of final Human policy distribution is equal to $L_H^{\CCVE}$      & 
$P=0.31$	& $t=+1.0$	& $d=+0.2$ & \\
\bf $H_0$: mean of final Machine policy distribution is equal to $L_M^{\CCVE}$    & 
$P=0.13$	& $t=-1.6$	& $d=-0.4$ & \\
\hline
Experiment 3 &&&& \\
\bf $H_0$: mean of initial Human action distribution is equal to $h^{\SE}$        & 
$P=0.27$	& $t=-1.2$	& $d=-0.4$ & \\
\bf $H_0$: mean of initial Machine action distribution is equal to $m^{\SE}$      & 
$P=0.33$	& $t=-1.0$	& $d=-0.3$ & \\
\bf $H_0$: mean of initial Human policy distribution is equal to $L_H^{\SE}$      & 
$P=1.00$	& $t=+0.0$	& $d=+1.0$ & \\
\bf $H_0$: mean of initial Machine policy distribution is equal to $L_M^{\SE}$    & 
$P=1.00$	& $t=+0.0$	& $d=\text{NaN}$ & \\
\bf $H_0$: mean of initial Machine cost distribution is equal to $c_M^{\SE}$      & 
$P=0.74$	& $t=-0.3$	& $d=-0.1$ & \\
$H_0$: mean of initial Human action distribution is equal to $h^{\RSE}$       & 
$P=0.00$	& $t=+7.9$	& $d=+2.6$ &  $\star$\\
$H_0$: mean of initial Machine action distribution is equal to $m^{\RSE}$     & 
$P=0.00$	& $t=+8.4$	& $d=+2.8$ &  $\star$\\
$H_0$: mean of initial Human policy distribution is equal to $L_H^{\RSE}$     & 
$P=0.00$	& $t=-\infty$	& $d=-\infty$ &  $\star$\\
$H_0$: mean of initial Machine policy distribution is equal to $L_M^{\RSE}$   & 
$P=0.00$	& $t=+\infty$	& $d=\text{NaN}$ &  $\star$\\
$H_0$: mean of initial Machine cost distribution is equal to $c_M^{\RSE}$     & 
$P=0.00$	& $t=+7.7$	& $d=+2.6$ &  $\star$\\
$H_0$: mean of final Human action distribution is equal to $h^{\SE}$          & 
$P=0.00$	& $t=-7.5$	& $d=-2.5$ &  $\star$\\
$H_0$: mean of final Machine action distribution is equal to $m^{\SE}$        & 
$P=0.00$	& $t=-11.9$	& $d=-4.0$ &  $\star$\\
$H_0$: mean of final Human policy distribution is equal to $L_H^{\SE}$        & 
$P=0.00$	& $t=+22.9$	& $d=+7.6$ &  $\star$\\
$H_0$: mean of final Machine policy distribution is equal to $L_M^{\SE}$      & 
$P=0.00$	& $t=-19.4$	& $d=-6.5$ &  $\star$\\
$H_0$: mean of final Machine cost distribution is equal to $c_M^{\SE}$        & 
$P=0.00$	& $t=-6.3$	& $d=-2.1$ &  $\star$\\
\bf $H_0$: mean of final Human action distribution is equal to $h^{\RSE}$         & 
$P=0.07$	& $t=+2.1$	& $d=+0.7$ & \\
\bf $H_0$: mean of final Machine action distribution is equal to $m^{\RSE}$       & 
$P=0.06$	& $t=+2.1$	& $d=+0.7$ & \\
\bf $H_0$: mean of final Human policy distribution is equal to $L_H^{\RSE}$       & 
$P=0.01$	& $t=-3.1$	& $d=-1.0$ &  $\star$\\
\bf $H_0$: mean of final Machine policy distribution is equal to $L_M^{\RSE}$     & 
$P=0.01$	& $t=+3.3$	& $d=+1.1$ &  $\star$\\
\bf $H_0$: mean of final Machine cost distribution is equal to $c_M^{\RSE}$       & 
$P=0.07$	& $t=+1.7$	& $d=+0.6$ & \\
\hline
\end{tabular}
\caption[Exact values from statistical tests]{%
Null hypotheses and exact values of statistics for $t$-tests used in Experiments 1, 2 and 3 ($P$-values, $t$ statistic, and Cohen's $d$ effect size).
All tests have degrees of freedom equal to 19.
Statistical significance ($\ast$) determined by comparing $P$-value with confidence threshold $0.05$.
Tests on actions and policies are $2$-sided, tests on costs are $1$-sided. The bold rows are outcomes predicted by the game theory analysis.
}
\label{tab:pvalues}
\end{table}
\newpage
\clearpage
\begin{center}
{\footnotesize
\begin{minipage}[t]{.65\textwidth}
\begin{mdframed}
\begin{lstlisting}
repeat:
  pick adaptation rate $\alpha$ and sign $s$ randomly
  initialize actions $h_0,m_0$ randomly
  for $t$ in $\{1,\dots,T\}$:
    $h_{t} = s$*get_manual_input($t$)
    display_cost($c_H(h_t,m_t)$)
    if $\alpha=0$:
      $m_{t+1} = \Nash{m}$  
    else if $0<\alpha<\infty$:                  
      $m_{t+1}=m_t - \alpha \partial_m c_M(h_t,m_t)$
    else if $\alpha=\infty$:
      $m_{t+1}=L_{M,0}h_t+\ell_{M,0}$
\end{lstlisting}
\end{mdframed}
\end{minipage}
}

\rcode{code:exp1}
\vspace{1ex}
\footnotesize Protocol S\ref{code:exp1}: { Algorithm description of Experiment 1.}
\label{protocol:exp1} 
\end{center}

\begin{center}
{
\footnotesize
\begin{minipage}[t]{.44\textwidth}
\begin{mdframed}
\begin{lstlisting}
function $\bf{run\_trial}$($L_M,\ell_M$):
  initialize $h_{0}$ randomly
  for $t$ in $\{1,\dots,T\}$:
    $h_{t}=$get_manual_input($t$) 
    $m_{t} = L_{M} h_{t} + \ell_{M}$
    display_cost($c_H(h_{t},m_{t})$)
  return median of $h_t$ and $m_t$
\end{lstlisting}
\vspace{6.1ex}
\end{mdframed}
\end{minipage}
\begin{minipage}[t]{.55\textwidth}
\begin{mdframed}

\begin{lstlisting}
initialize $L_{M,0}$ and $\ell_{M,0}$
for $k$ in $\{0,\dots,K-1\}$:
  $(\est{h},\est{m})\leftarrow\bf{run\_trial}$($L_{M,k},\ell_{M,k}$):
  $(\est{h}',\est{m}')\leftarrow\bf{run\_trial}$($L_{M,k},\ell_{M,k}+\delta$):
  $\est{L}_{H,k+1}=({\widetilde h'-\widetilde h})/({\widetilde m'-\widetilde m})$  
  $L_{M,k+1}=-(B_M+\est{L}_{H,k+1}^\tops D_M)/(A_M+\est{L}_{H,k+1}^\tops B_M^\tops)$
  $\ell_{M,k+1}\,\,=-(b_M+\est{L}_{H,k+1}^\tops d_M)/(A_M+\est{L}_{H,k+1}^\tops B_M^\tops)$
end experiment
\end{lstlisting}
\end{mdframed}
\end{minipage}
}

\rcode{code:exp2}
\footnotesize Protocol S\ref{code:exp2}: Algorithm description of Experiment 2.
\label{protocol:exp2} 
\end{center}

\begin{center}
{\footnotesize
\begin{minipage}[t]{.44\textwidth}
\begin{mdframed}
\begin{lstlisting}
function $\bf{run\_trial}$($L_M, h_M^*,m_M^*$):
  initialize $h_0$ randomly  
  for $t$ in $\{1,\dots,T\}$:
    $h_{t}=$get_manual_input($t$)
    $m_{t} = L_{M} (h_t-h_M^\ast) + m_M^\ast$
    display_cost($c_H(h_t,m_t)$)
  return median of $c_M(h_t,m_t)$ 
  \end{lstlisting}
  \vspace{1.15ex}
 \end{mdframed}
\end{minipage}
\begin{minipage}[t]{.55\textwidth}
\begin{mdframed}
\begin{lstlisting}
initialize $L_{M,0}$ and $(m_M^\ast, h_M^\ast)$
for $k$ in $\{0,\dots,K-1\}$:
  $\est{c_M}\leftarrow $ $\bf{run\_trial}$($L_{M,k},h_M^*,m_M^*$)
  $\est{c_M}'\leftarrow $ $\bf{run\_trial}$($L_{M,k}+\Delta,h_M^*,m_M^*$)
  grad_M$\,=({ \est{c_M}^{'} - \widetilde{c_M}})/{\Delta}$
  $L_{M,k+1} = L_{M,k} - \gamma*$grad_M
end experiment
\end{lstlisting}
\end{mdframed}	
\end{minipage}
}

\rcode{code:exp3}
\footnotesize Protocol S\ref{code:exp3}: Algorithm description of Experiment 3.
\label{protocol:exp3} 
\end{center}

\begin{table}[ht]
\small
\centering
\begin{tabular}{|c|c|l|}
\hline
human $H$ & machine $M$ & \\
\hline
$\ActH=[-1,1]\subset\R$
& $\ActM=\R$
& player action spaces \\
$\actH\in\ActH$ & $\actM\in\ActM$ & player actions \\
$\costH:\ActH\times\ActM\into\R$ & $\costM:\ActH\times\ActM\into\R$ & player costs \\
\hline
\end{tabular}
\caption[Symbols and terminology for the co-adaptation game]{\label{tab:symbols}%
Symbols and terminology for the co-adaptation game between human and machine.%
}
\vspace{2em}
\centering
\begin{tabular}{|l|l|}
\hline
Symbol & Description   \\
\hline 
$T>0$ & time horizon  \\
$t\in\{0,1,\dots,T\}$ & time (discrete steps)\\
$h_t\in \ActH = [-1,1]$ & $H$'s action at time $t$\\
$m_t\in\ActM=\R$ & $M$'s action at time $t$  \\
$c_H(h_t,m_t)\in\R$ & $H$'s cost at time $t$  \\
$c_M(h_t,m_t)\in\R$ & $M$'s cost at time $t$ \\
\hline 
\hline
Experiment 1: &\\
$\alpha\in[0,\infty]$ & $M$'s adaptation rate  \\
$\partial_m c_M(h,m)\in\R$ & derivative of $M$'s cost with respect to $m$\\
 $L_{M,0}(\cdot)+\ell_{M,0}\in\R\to\R$ & $M$'s Nash policy\\
$(\Nash{h},\Nash{m})\in \ActH\times\ActM$ & Nash equilibrium \\
 $(\Stack{h},\Stack{m})\in \ActH\times\ActM$  & 
 human-led Stackelberg equilibrium \\
\hline\hline
Experiment 2: &\\
$k\in\{0,\dots,K\} $ & conjectural variation iteration   \\
$\delta\in\R$ & perturbation to constant term of $M$'s policy \\
$\est{L}_{H,k}\in\R$ & $M$'s estimate of $H$'s policy slope at iteration $k$ \\
 $L_{M,k}(\cdot)+\ell_{M,k}\in\R\to\R$ & $M$'s policy at iteration $k$ \\
 $({h}^\CCVE,{m}^\CCVE)\in \ActH\times\ActM$  &  consistent conjectural variations equilibrium \\
\hline\hline 
Experiment 3: &\\
$k\in\{0,\dots,K\}$ & policy gradient iteration  \\
$\Delta\in\R$ & perturbation to slope term of $M$'s policy  \\
$\partial_{L_M}\est{c}_{M}(L_M)\in\R$ &
$M$'s policy gradient estimate \\
 $L_{M,k}(\cdot)+\ell_{M,k}\in\R\to\R$ & $M$'s policy at iteration $k$ \\
 $(\RStack{h},\RStack{m})\in \ActH\times\ActM$  &  machine-led reverse Stackelberg equilibrium \\
 $(\Opt{h}_M,\Opt{m}_M)\in \ActH\times\ActM$  &  $M$'s global minimum \\
\hline
\end{tabular}
\caption[Symbols and terminology for the three experiments]{Symbols and terminology for the game used in the three experiments.}
\end{table}

\clearpage
\subsection{Numerical simulations}
\label{sec:supp-sim}
To provide simple descriptive models for the outcomes observed in each of the three Experiments, numerical simulations were implemented using 
Python 3.8~\cite{rossum2009python}.
The shared parameter, cost and gradient  definitions are included in Sourcecode S0.

\paragraph{Experiment 1}
To predict what happens in the range of adaptation rates between the two limiting cases (i.e.\ for $0 < \alpha < \infty$), a simulation of the human's behavior was implemented based on approximate gradient descent. 
The model of the human simply uses finite differences to estimate the derivative of its cost ($\costH$) with respect to its action ($h$) and then adapts its action to descend this cost gradient. 
Importantly, it is assumed that the human performs these derivative estimation and gradient descent procedures 
slower than the machine, i.e.\ the human takes one gradient step for every {$K$} 
machine steps.
Since the machine's steps occur at a rate of $60$ samples per second, this timescale difference corresponds to the human taking steps at a rate of {$60/K$} 
samples per second. 
The Python code for simulating Experiment 1 is included in Sourcecode S1.

\noindent
\paragraph{Experiment 2}
 
To predict what happens when the machine perturbs the constant term of its policy and uses the outcome to estimate of the human's policy slope, 
a simulation of their behavior was implemented based on the
conjectural variations iteration.
The machine best responds to the human's policy.
The model of the human uses
the derivative of its cost ($\costH$) assuming that the machine's action ($m$) is related to its own action ($h$) by conjectural variation ($L_{M,k}$) and then adapts its action to descent this cost gradient.
It is assumed that the 
machine observes the human and machine's actions to compute the estimate of the human's policy slope ($\est{L}_{H,k}$).
The Python code for simulating Experiment 2 is included in Sourcecode S2.

\paragraph{Experiment 3}

To predict what happens when the machine perturbs the linear term of its policy, a simulation was implemented based on policy gradient.
The model of the human is the same as the previous simulation of Experiment 2. 
The machine uses the gradient estimate of the observed cost, and does not require observe the human's action or policy as was required in the previous experiment.
The Python code for simulating Experiment 3 is included in Sourcecode S3.

\newpage

\begin{center}
{ \scriptsize
\begin{lstlisting}[language=Python,frame=single,xleftmargin=.15\textwidth,xrightmargin=.15\textwidth]
T = 10000   # time samples

# human's cost parameters
AH, BH, DH, hH, mH = 1, -1/3, 7/15, 1/10, 7/10   
  
# machine's cost parameters
AM, BM, DM, hM, mM = 1, -1, 2, 0, 0 

def cost_H(h, m):     # H's cost
    return AH*(h-hH)**2/2 + (h-hH)*BH*(m-mH) + DH*(m-mH)**2/2

def cost_M(h, m):     # M's cost
    return AM*(m-mM)**2/2 + (h-hM)*BM*(m-mM) + DM*(h-hM)**2/2
    
def grad_H(h, m, LM): # H's gradient
    return AH*(h-hH) + BH*(m-mH) + LM*(BH*(h-hH) + DH*(m-mH))

def grad_M(h, m, LH): # M's gradient
    return AM*(m-mM) + BM*(h-hM) + LH*(BH*(h-hH) + DH*(m-mH))

def ceil(x):
    return int(x) if int(x)==x else int(x+1)
\end{lstlisting}
}
Sourcecode S0: Definitions of parameters, cost functions and gradients of the two players.
\label{sourcecode:exp0} 
\end{center}

\begin{center}
{ \scriptsize
\begin{lstlisting}[language=Python,frame=single,xleftmargin=.1\textwidth,xrightmargin=.1\textwidth]
# machine's adaptation rates
alphas = [3*10**(i/10) for i in range(-29,-9)]
beta = 0.003 # human's adaptation rate (assumed)
delta = 1e-5 # perturbation size of constant term of H's policy

results = []
for alpha in alphas:
    K = ceil(alpha/beta)          # ratio of M iterations to H iterations
    N = ceil(T/K)*K+1             # number of total iterations 
    h,m = [0]*N, [0]*N            # initialize actions
    
    for t in range(0, T, K):      # gradient descent loop
        c_H = []                  # H's observed cost
        
        for d in [delta, 0]:    
        
            for k in range(t, t+K):
                # perturb H's action
                h[k]   = h[t] + d  
                # update M's action
                m[k+1] = m[k] - alpha*grad_M(h[k],m[k],0) 
            c_H.append(cost_H(h[k],m[k]))
            
        gradH = (c_H[0]-c_H[1])/2/delta # estimate H's gradient
        
        h[t+K] = h[t] - K*beta*gradH    # update H's action
        m[t+K] = m[k+1]
    results.append([h[-1],m[-1]])
\end{lstlisting}
}
Sourcecode S1: Numerical simulation of Experiment 1. 
\label{sourcecode:exp1} 
\end{center}
\newpage

\begin{center}
{ \scriptsize
\begin{lstlisting}[language=Python,frame=single,xleftmargin=.15\textwidth,xrightmargin=.15\textwidth]
K = 10       # total conjectural variations iterations
delta = 1e-1 # perturbation size (of constant term of M's policy)

h,m = [0]*(K*T+1), [0]*(K*T+1) # initialize actions 
LH,LM = [0]*(K+1), [0]*(K+1)   # initialize policy slopes
LM[0] = -BM/AM                 # initialize M's policy

# conjectural variations iteration loop
for k in range(K):
    h_, m_ = [], []       # steady state actions
    
    for d in [delta,0]:   # run a pair of trials
    
        for t in range(k*T, k*T + T):
            # update H's action
            h[t+1] = h[t] - beta*grad_H(h[t], m[t], LM[k])
            # update M's action
            m[t+1] = LM[k]*(h[t]-hM) + mM + d
            
        h_.append(h[t+1])
        m_.append(m[t+1])
        
    # estimate H's policy slope 
    LH[k+1] = (h_[1] - h_[0])/(m_[1] - m_[0])
    
    # update M's policy slope
    LM[k+1] = -(BM + LH[k+1]*DM)/(AM + LH[k+1]*BM)	
\end{lstlisting}
}
Sourcecode S2: Numerical simulation of Experiment 2.
\label{sourcecode:exp2} 
\end{center}

\begin{center}
{ \scriptsize
\begin{lstlisting}[language=Python,frame=single,xleftmargin=.15\textwidth,xrightmargin=.15\textwidth]
K = 10       # total policy gradient iterations
Delta = 1e-1 # perturbation size (of slope term of M's policy)
beta = 3e-3  # human's learning rate
gamma = 2    # policy gradient step size

# initialize actions and policies
h,m = [0]*(K*T+1), [0]*(K*T+1) # initialize actions 
LH,LM = [0]*(K+1), [0]*(K+1)   # initialize policy slopes
LM[0] = -BM/AM

# policy gradient loop
for k in range(K):                  
    c_M = []                    # M's steady state cost
    
    for D in [Delta, 0]:    # run pair of trials
    
        for t in range(k*T, k*T+T):
            # update H's action
            h[t+1] = h[t] - beta*grad_H(h[t], m[t], LM[k] + D)
            # update M's action
            m[t+1] = (LM[k] + D)*(h[t] - hM) + mM
            
        c_M.append(cost_H(h[t],m[t]))
        
    # estimate M's policy gradient
    gradM = (c_M[0] - c_M[1])/Delta/2
    
    # update M's policy slope
    LM[k+1] = LM[k] - gamma*gradM
\end{lstlisting}
}
Sourcecode S3: Numerical simulation of Experiment 3.
\label{sourcecode:exp3} 
\end{center}

\clearpage 
\section{Task load survey and feedback forms}\label{sec:feedback}
Each participant filled out a task load survey and optional feedback form upon finishing an experiment.

\subsection{Task load survey}
The NASA Task Load Index\cite{hart1988development}
 was used to assess participant's mental, physical, and temporal demand while performing the task.
The questions asked are:

\noindent
 {\bf 1. Mental Demand: }
          How mentally demanding was the task?
          
          Very Low (-10) -- 
          Very High (10)
          
 \noindent
  {\bf 2. Physical Demand:}
          How physically demanding was the task?

          Very Low (-10) -- 
          Very High (10)
          
 \noindent
{\bf 3. Temporal Demand:}
          How hurried or rushed was the pace of the task?
         
          Very Low (-10) -- 
          Very High (10)
          
\noindent
 {\bf 4. Performance:}
          How successful were you in accomplishing what you were asked to do?
          
          Perfect (-10) --
          Failure (10)
          
\noindent
 {\bf 5. Effort:}
          How hard did you have to work to accomplish your level of performance?
          
          Very Low (-10) -- 
          Very High (10)

\noindent
 {\bf 6. Frustration:}
          How insecure, discouraged, irritated, stressed, and annoyed were you?

          Very Low (-10) -- 
          Very High (10)
          
 \noindent
Table \ref{table:survey} provides the data from the survey for all participants.

\begin{table}[ht]
\centering
\begin{tabular}{|lrrr|}
\hline
 &  25\% quartile &  median &  75\% quartile \\
 \hline
 Mental Demand      &  -8 &  -5 &   0 \\
 Physical  Demand  &  -9 &  -6 &  -2 \\
 Temporal  Demand   &  -8 &  -5 &  -1 \\
 Performance &  -9 &  -6 &  -2 \\
 Effort      &  -6 &  -2 &   3 \\
 Frustration &  -9 &  -4 &   2 \\
\hline
 \end{tabular}
 \caption[Task load survey results]{Results from the task load survey for three experiments under two game costs with 20 participants per experiment, totalling 120 participants. }
 \label{table:survey}
 \end{table}
\newpage
\clearpage
\subsection{Optional Feedback}

Additional feedback was optionally provided by participants.

\noindent {\bf  Any feedback? Let us know here: } 
[Text box]

\noindent
Table \ref{table:feedback} provides the feedback submitted by participants.

\begin{table}[ht]
{
\footnotesize
\begin{tabular}{lp{30em}}
\hline
                         Experiment &                                                                                                                                                                Feedback \\
\hline
    Experiment 1 (quadratic) &                                                                                                         None, keep up the good work and thank you for the oportunity :) \\
    Experiment 1 (quadratic) &                                                                                                                                                               cool test \\
    Experiment 1 (quadratic) &                                                           I think that the study was very different from other studies I have taken in Prolific. More challenging too.  \\
    Experiment 1 (quadratic) &                                                                                                                                                   Everything was fine!! \\
    Experiment 1 (quadratic) &                                                                                                  The "keep this small task" was abusable if you kept your cursor still. \\
    Experiment 1 (quadratic) &                                                                                                                    Everything worked perfectly, thanks for inviting me! \\
    Experiment 2 (quadratic) &                                                                                                                                                                     No  \\
    Experiment 2 (quadratic) &                              The experiment was interesting, it was a bit frustrating when the option to fill the block moved too fast before i could do it accordingly \\
    Experiment 2 (quadratic) &                                                                                                                                                                     N/A \\
    Experiment 2 (quadratic) &                                                                                                                                     In my opinion the task was easy \\
    Experiment 2 (quadratic) &                                                                                                                                   It was an interesting task! thank you \\
    Experiment 3 (quadratic) &                                                                                                       It was an interesting study that I would love to partake in again \\
    Experiment 3 (quadratic) &                                                                                                                                                                      NA \\
    Experiment 3 (quadratic) &                                                                                                                                                       I liked  the task \\
    Experiment 3 (quadratic) &                                                                                                                            The survey was easy, it just required focus. \\
    Experiment 3 (quadratic) &                                                                                                                                       too much time needed for the task \\
Experiment 1 (non-quadratic) &                                                                                                                                            I think that human's eye is  \\
Experiment 1 (non-quadratic) &                                                                                                                                     The study was okay, but a bit slow. \\
Experiment 1 (non-quadratic) & It Would been better, if it was more detail in explaining and to be able to lick when you have the box at the smallest size possible, thanks once again for the study   \\
Experiment 1 (non-quadratic) &                                                                                                                                    this gave me anxity but it was good  \\
Experiment 2 (non-quadratic) &                                                                                                                                  Maybe some instructions would be nice  \\
Experiment 2 (non-quadratic) &                                                                                                 I didn’t understand the aim of the study, but it’s always nice to play  \\
Experiment 2 (non-quadratic) &                                                                                                                                                            No feedback  \\
Experiment 2 (non-quadratic) &                                                                                                                                                 Everything was perfect. \\
Experiment 2 (non-quadratic) &                                                                                                                                                                      NA \\
Experiment 2 (non-quadratic) &                                                                               not sure why the waiting time for the next task during the 20 exercises  but it was good  \\
Experiment 2 (non-quadratic) &                                                                                        At first i didn't notice that the breaks were timed, made me fail  couple tasks. \\
Experiment 3 (non-quadratic) &                                                     Either instructions were unclear or the time between tasks was WAY too long. Unless that was part of the study.. :O \\
\hline
\end{tabular}
}
\caption[Feedback results]{Written feedback from participants. Optionally provided.}
\label{table:feedback}
\end{table}
\clearpage

\clearpage
\newgeometry{top=0.5in,bottom=.75in}

\begin{figure*}[!ht]
\centering
\includegraphics[width=\textwidth]{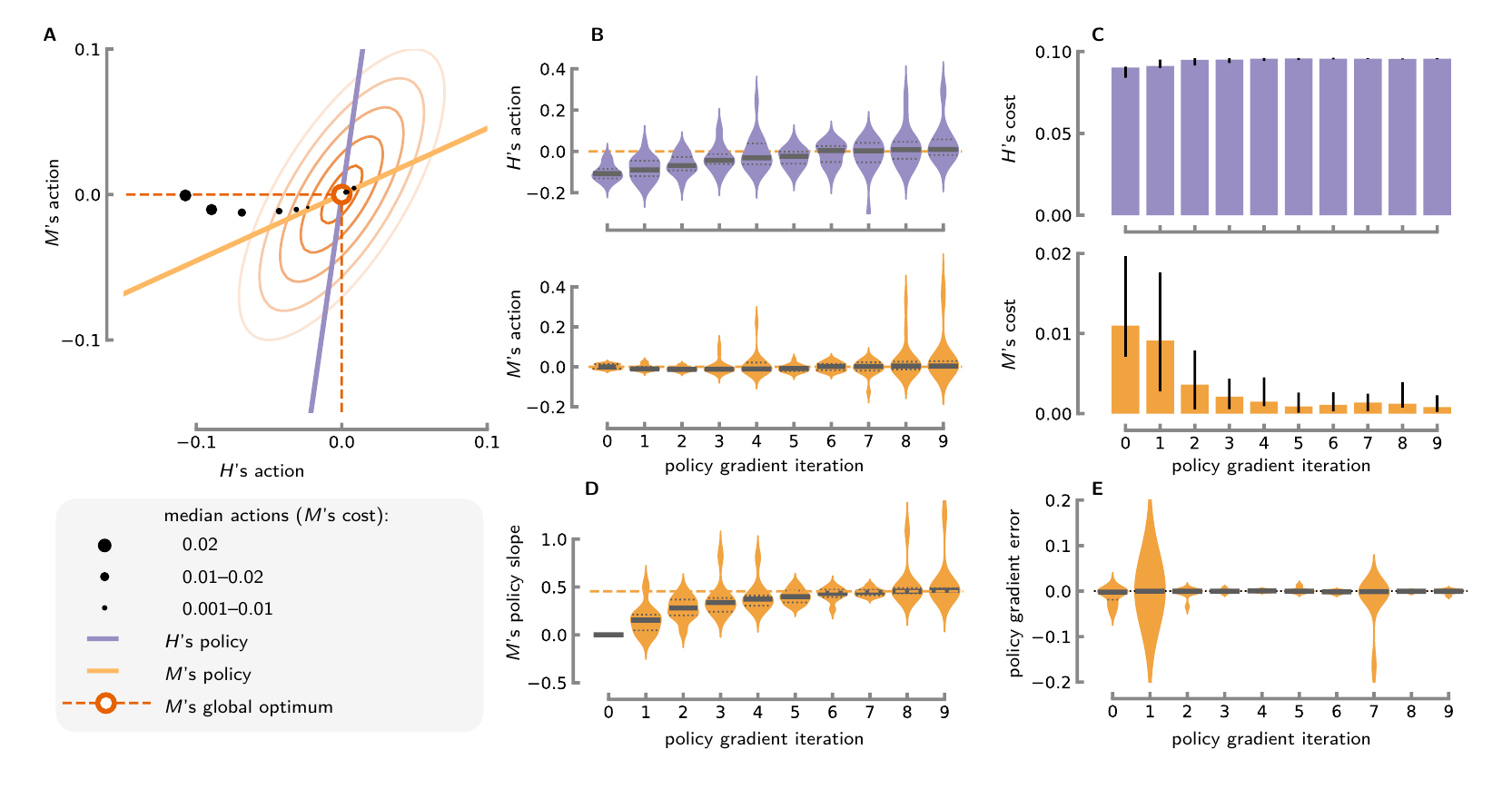}
  \caption[Experiment 3 with different initial policy]{\label{fig:exp3-other-init}
  \footnotesize
  \textbf{Experiment 3 with different initial policy} ($n=20$): gradient descent in policy space for a different initial machine policy.
  (\textbf{A}) Game-theoretic equilibria and best-response functions.
  (\textbf{B}) Decision vector distributions.
  (\textbf{C}) Cost distributions.
  (\textbf{D}) Machine policy slopes.
  (\textbf{E}) Estimation error of machine policy gradients.
  Action IQR in (B) contains the machine's minimum at each iteration 4 to 9.
  Machine's policy slope distribution IQR in (D) reaches the theoretically-predicted slope that would yield the machine's minimum as the game outcome.
  The machine's policy gradient IQR in (E) contains the theoretical gradient at every policy gradient iteration.
  }
\end{figure*}

\begin{figure*}[!ht]
\centering
\includegraphics[width=\textwidth]{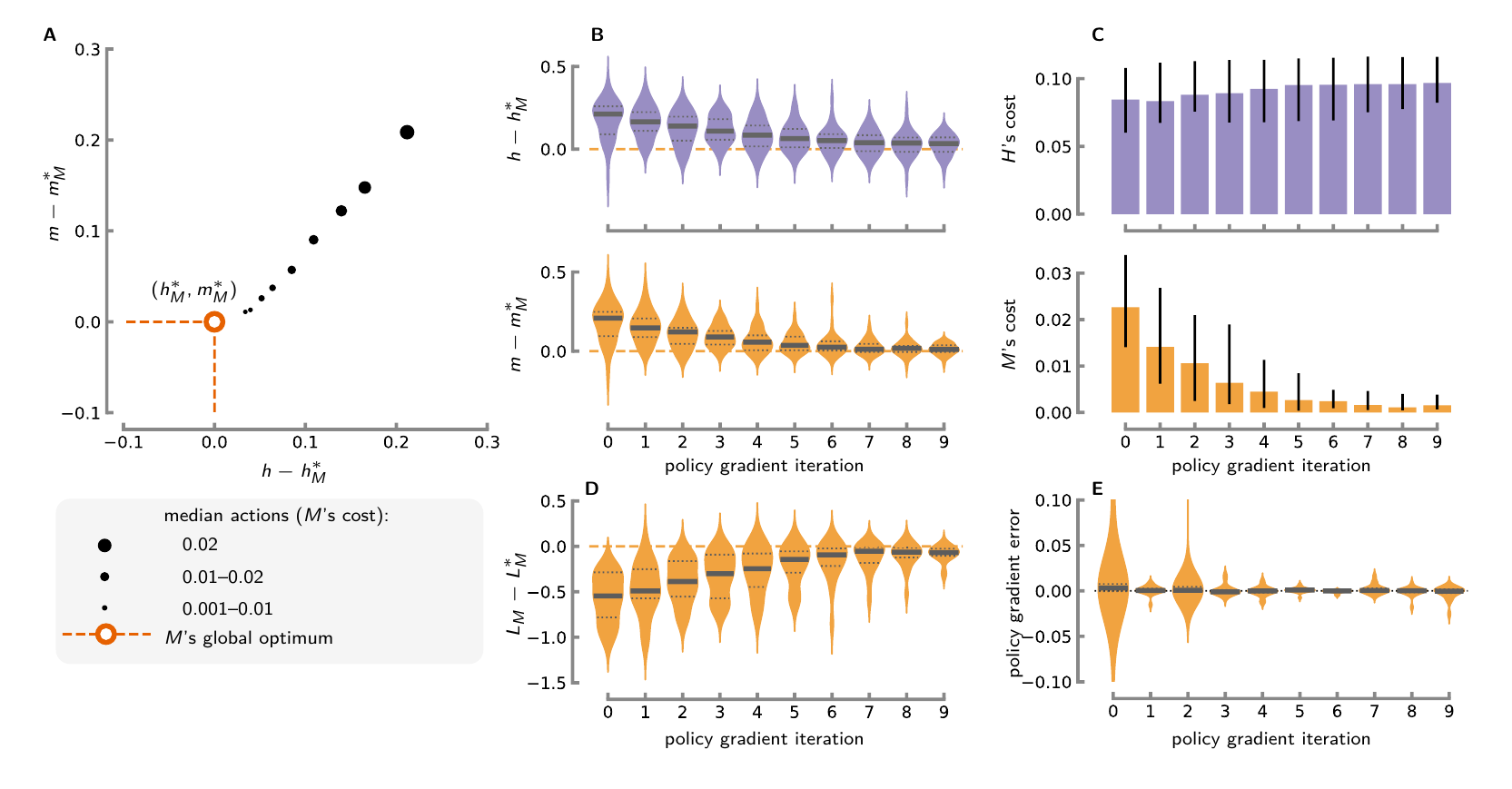}
  \caption[Experiment 3 with different machine optima]{\label{fig:exp3-other-optima}
  \footnotesize
  \textbf{Experiment 3 with different machine optima} ($n=18$): gradient descent in policy space for differing machine optima.
  (\textbf{A}) Game-theoretic equilibria and best-response functions.
  (\textbf{B}) Decision vector distributions.
  (\textbf{C}) Cost distributions.
  (\textbf{D}) Machine policy slopes.
  (\textbf{E}) Estimation error of machine policy gradients.
  Action IQR in (B) contains the machine's minimum at each iteration 7 to 9.
  Machine's policy slope distribution IQR in (D) approaches the theoretically-predicted slope that would yield the machine's minimum as the game outcome.
      }
\end{figure*}

\begin{figure*}[!ht]
\centering
\includegraphics[width=\textwidth]{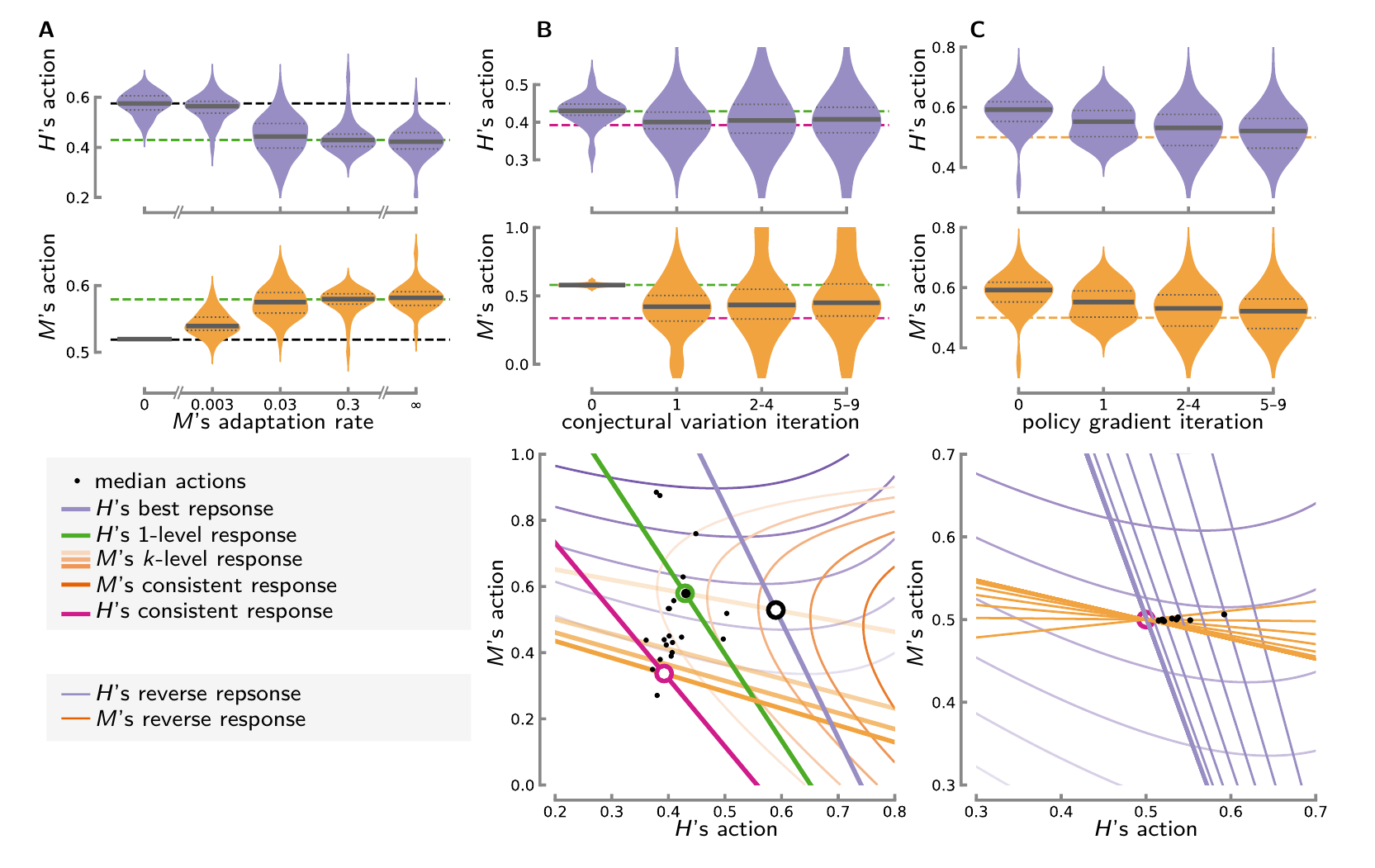}
\caption[Modified Experiments 1, 2 and 3 with non-quadratic costs]{\label{fig:exp-cobb}
\footnotesize
\textbf{Modified Experiments 1, 2 and 3 with non-quadratic costs} ($n=20\times 3$): Non-quadratic costs.
    (\textbf{A}) Gradient descent in action space; decision vector distributions.
(\textbf{B}) Conjectural variation in policy space; decision vector distributions, game theoretic equilibria and best-response functions.
(\textbf{C}) Gradient descent in policy space; decision vector distributions and policy gradient iterations.   
}
\end{figure*}

\begin{figure*}[!ht]
\centering
\includegraphics[width=\textwidth]{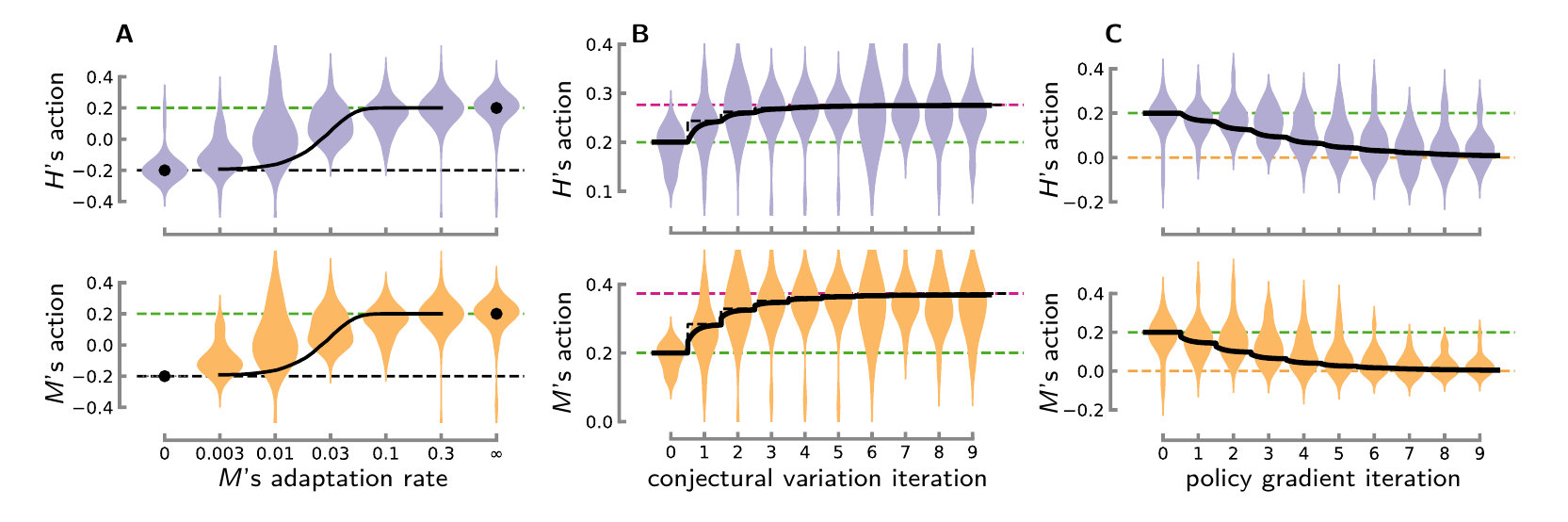}
  \caption[Simulations of Experiments 1, 2 and 3]{\label{fig:sim-quad}
  \footnotesize
  \textbf{Simulations of Experiments 1, 2 and 3}: Solid lines and dots are the simulation data, overlaid on violin plots from main paper. Dashed lines are analytical predictions.
  (\textbf{A}) Gradient descent in action space.
  (\textbf{B}) Conjectural variation in policy space.
  (\textbf{C}) Gradient descent in policy space.
  }
\vspace{10em}
\end{figure*}

\begin{figure*}[!ht]
\centering
\includegraphics[width=\textwidth]{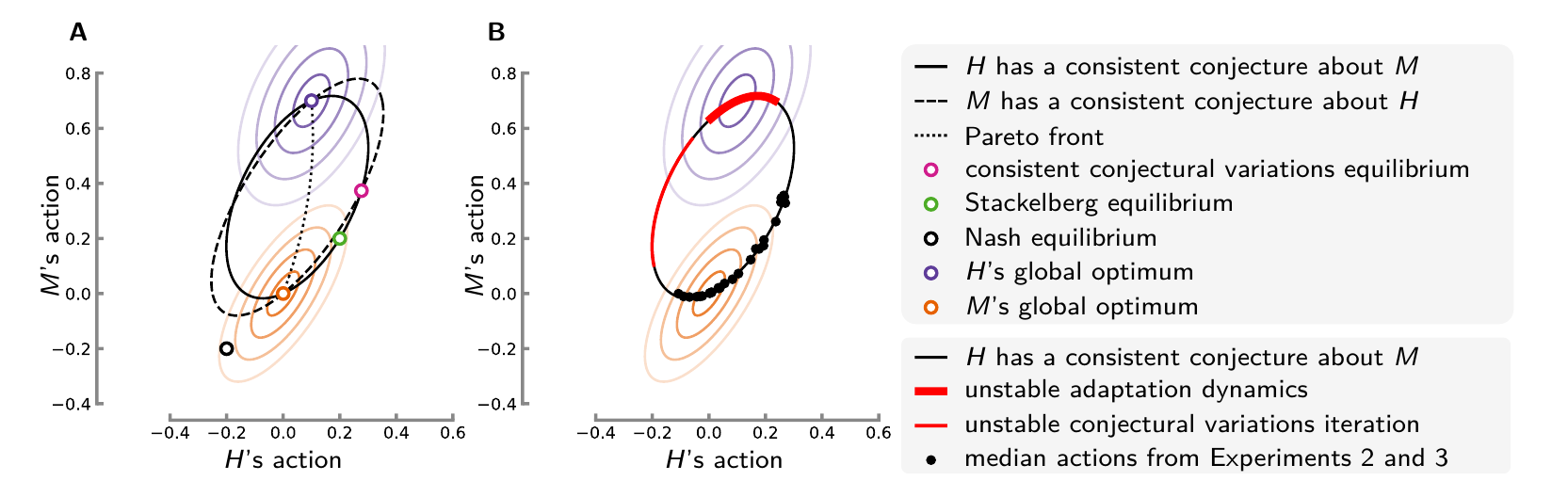}
  \caption[Comparing Pareto optimality with conjecture consistency]{\label{fig:pareto-consistency}   \scriptsize
  \textbf{Comparing Pareto optimality with conjecture consistency}
    \textbf{(\sf A)} 
    The analytical solution for the continuum of equilibria where the human has a consistent conjecture about the machine and vice versa, compared with the Pareto optimal points.
  \textbf{(\sf B)} The median actions from Experiments 2 and 3 coincide with the ellipse that corresponds to the human having a consistent conjecture about the machine. 
  }
\vspace{30em}
\end{figure*}

\restoregeometry
\baselineskip24pt

\newpage

\end{document}